\documentclass{article} 
\usepackage{iclr2027_conference,times}
\iclrfinalcopy

\usepackage{amsmath,amsfonts,bm}

\def\eqref#1{equation~\ref{#1}}

\def\1{\bm{1}}

\DeclareMathAlphabet{\mathsfit}{\encodingdefault}{\sfdefault}{m}{sl}
\SetMathAlphabet{\mathsfit}{bold}{\encodingdefault}{\sfdefault}{bx}{n}

\usepackage{hyperref}
\usepackage{url}
\usepackage{amsfonts}
\usepackage{amsmath}   
\usepackage{amsthm}    
\usepackage{amssymb}   
\usepackage{graphicx}
\usepackage{booktabs} 
\usepackage{hyperref}
\usepackage{booktabs}
\usepackage{ragged2e}
\usepackage{microtype}
\usepackage{array}
\usepackage{enumitem}
\usepackage{placeins}
\usepackage{float}
\usepackage{graphicx}
\usepackage{subcaption} 
\usepackage{multirow}
\usepackage{tabularx}
\usepackage{siunitx}
\usepackage{algorithm,algpseudocode}

\newtheorem{assumption}{Assumption}[section]
\newtheorem{proposition}{Proposition}[section]
\newtheorem{definition}{Definition}[section]

\title{Beyond Model Ranking: Regime Diagnosis for Distributional-Statistical Misspecification in Industrial Time-Series Forecasting}

\author{Pengyu Nie \\
JD.com, Inc. \\
\texttt{niepengyu.3@jd.com} \\
Chenglang XU \\
JD.com, Inc. \\
\texttt{xuchenglang1@jd.com} \\
Yaoshi Chen  \\
JD.com, Inc. \\
\texttt{chenyaoshi1@jd.com} \\
Chaogan Ren \\
JD.com, Inc. \\
\texttt{renchaoqan@jd.com} \\
Wei Hu \\
JD.com, Inc. \\
\texttt{huwei157@jd.com} \\
Chao Yang \\
JD.com, Inc. \\
\texttt{yangchao357@jd.com} \\
Jiangong Zhang \\
JD.com, Inc. \\
\texttt{zhangjiangong@jd.com} \\
}

\begin{document}

\maketitle

\begin{abstract}
Time-series forecasting models achieve strong benchmark performance but
exhibit severe systematic bias in industrial deployments. This train--deploy
gap is conventionally attributed to temporal-structural errors or
distribution shifts. We characterize a complementary source that these
explanations overlook: canonical losses embed fixed statistical priors, while
industrial demand mixes benign and pathological regimes---zero-inflation,
skewness, high variability---in which these priors are systematically
violated. The induced bias persists even under perfect temporal modeling,
remains in a distributional-shape component that normalization cannot
remove, and creates an aggregation trade-off invisible to aggregate metrics.
We turn these observations into an evaluation toolkit centered on the
\textbf{Regime-wise Relative Bias Vector (RBV)}: a metric-agnostic,
regime-decomposed diagnostic that audits how pooled training allocates
systematic mismatch across pathological subpopulations. A controlled
attribution analysis decomposes RBV into a model-independent intrinsic
floor, set by each loss's estimand, and an excess component attributable
to training, tracing observed bias to the loss rather than the model. A
large-scale study---$13$ loss objectives, $3$ seeds, $60{,}000$+ series
spanning RetailShiftBench and M5, with random-split controls---shows that
regime-aware diagnosis separates optimization-type from bias-type failure,
and that regime-aware training resolves the pooling-induced bias that
capacity scaling cannot, for mean-type losses. A formal structural observation, that risk under
evaluation-distribution contamination is affine in the pathology mixture
weight, grounds these findings. Our work complements model ranking with
mechanism-grounded, regime-oriented evaluation.
\end{abstract}

\vspace{-10pt}
\section{Introduction}
\vspace{-5pt}
Industrial demand forecasting underpins billion-dollar inventory and
supply-chain decisions across retail, energy, and manufacturing.
Despite strong performance on academic benchmarks, deep learning and
time-series foundation models exhibit severe systematic bias in real-world
deployments. The direction of this bias is dictated by the training loss:
mean-targeting objectives overshoot zero-demand periods, while median-type
objectives suppress realized bursts---so aggregate accuracy certifies
neither the presence nor direction of this mismatch on intermittent, skewed, and
high-variability sequences.

This prevalent train--deploy gap is conventionally attributed to
limited model capacity, imperfect temporal learning, or exogenous
distribution shifts. We surface a largely overlooked error source,
orthogonal to these explanations: canonical loss functions embed rigid
statistical priors---mean-seeking, symmetric, continuous-support
assumptions---that are systematically violated under industrial demand
pathologies (zero-inflation, skewness, scale variability). The two
canonical losses fail in opposite directions: mean-targeting MSE
over-predicts sparse regimes, whereas median-targeting MAE
under-predicts them (Tab.~\ref{tab:regime_perf}). Industrial datasets
are intra-dataset heterogeneous---benign, zero-inflated, skewed, and
high-CV segments coexist---so no single prior fits all data, and
globally-optimal losses (e.g., Tweedie chosen for zero-inflated sales)
trade performance across regimes rather than resolving the mismatch.
Because aggregate metrics average across regimes, reported benchmark
gains rest heavily on benign segments while deployment-relevant
mismatches stay hidden. Existing research offers no remedy from either
side. The OR community studies lumpy-demand pathologies (Croston's
family, SBA, Bayesian taxonomy models) yet treats them as modeling
obstacles, not evaluation dimensions for general-purpose forecasters.
Sliced benchmarks (TIME, TempusBench, TSFM-Bench) mitigate averaging
bias but examine only temporal-structural signals, so this
misspecification remains invisible to them; industrial evaluations
instead rely on static per-SKU grouping, which assigns each series a
single distributional profile and conceals coexisting pathologies.
Building on classical loss--estimand theory, we formalize the resulting
distributional-statistical misspecification as mechanistically
distinct from temporal-structural error: its bias persists under
idealized temporal modeling and stationary inputs. We claim
separability, not independence: loss-rooted bias remains non-vanishing
even when temporal-structural error is eliminated by construction---regime membership, not model quality alone, dictates the direction
of mismatch.

Based on this separation, we propose a mechanism-grounded diagnostic
paradigm that shifts evaluation from pure model ranking to
regime-oriented misspecification diagnosis.
Rather than chasing marginal accuracy improvements from novel network
architectures, this work bridges OR-driven industrial insights and modern
deep forecasting evaluation.
Our contributions are fourfold:
\begin{enumerate}[leftmargin=*, itemsep=1pt, topsep=2pt, parsep=0pt]
\setlength{\itemindent}{0pt}
\item \textbf{Pathology Taxonomy and Benchmark:} We construct a
quantifiable hierarchical regime taxonomy of orthogonal statistical
pathologies and their couplings, introduce \textbf{RetailShiftBench}
with full pathological-regime coverage, and audit mainstream benchmarks
against a pathology-coverage standard.
\item \textbf{Regime-Decomposed Bias Audit:} We propose a metric-agnostic
diagnostic protocol centered on the Regime-wise Relative Bias Vector
(RBV), which audits pooled training's allocation of systematic mismatch
across pathological subpopulations---imbalances invisible to aggregate
metrics. A controlled attribution analysis decomposes RBV into a
model-independent intrinsic floor, set by each loss's estimand, and an
excess component attributable to training, tracing observed bias to the
loss rather than the model.
\item \textbf{Supporting Structural Analysis:} Formal definitions
separate the two bias modes, a loss--estimand mapping links pathologies
to violated assumptions, and a formal observation shows that risk under
evaluation-distribution contamination is affine in the pathology mixture
weight; together with the accompanying boundary-condition analyses, this provides the mechanistic
basis for the empirical findings.
\item \textbf{Large-Scale Controlled Study:} We instantiate the protocol
over $13$ losses, $3$ seeds, and $60{,}000$+ series (RetailShiftBench;
external validation on M5), with random-split controls. Regime-aware
diagnosis consistently separates optimization-type from bias-type
failure, and regime-aware training resolves the pooling-induced bias it
exposes.
\end{enumerate}

\vspace{-10pt}
\section{Related Work}
\vspace{-5pt}
\subsection{Time-Series Diagnostic Framework Ecosystems}
\vspace{-5pt}
\begin{figure}[t]
\centering
\vspace{-5pt}
\includegraphics[width=\linewidth]{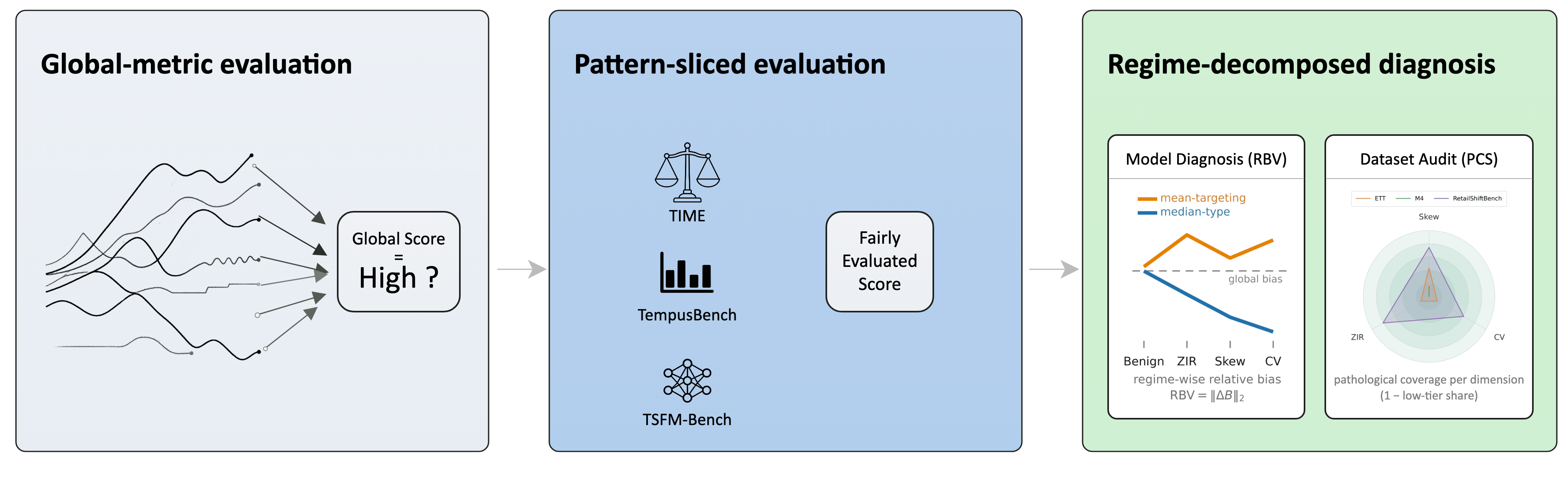}
\vspace{-10pt}
\caption{Evolution of time-series forecasting evaluation paradigms.}
\label{fig:paradigm_evolution}
\vspace{3pt}
\parbox{\linewidth}{\scriptsize
Evolution of time-series forecasting evaluation paradigms. Global aggregate metrics (left) hide regime-specific failure; pattern-sliced benchmarks (middle) retain uniform averaging over label-side pathologies; our RBV model diagnosis and PCS dataset audit (right) evaluate models and datasets at the level of distributional regimes.}
\vspace{-10pt}
\end{figure}

Conventional time-series benchmarks rely on idealized Gaussian stationary
assumptions and fail to capture authentic industrial distributional traits,
yielding deployment-irrelevant validation and hidden robustness risks. Recent
works move beyond global ranking toward fine-grained sliced evaluation:
TIME~\cite{qiao2026s} introduces pattern-aware slicing for temporal-structural
failures, while TempusBench~\cite{goktas2026tempusbench} and TSFM-Bench~\cite{li2025tsfm}
unify evaluation across diverse architectures and observe unstable scaling-law
behaviors. Yet their slicing axes remain temporal-structural,
leaving label-side pathologies unexamined.

Sliced and stratified evaluation are well-established in fairness and
robustness research to expose subgroup-level performance
gaps~\cite{d2022spotlight}; dataset-auditing practices further stress
systematic data quality assessment~\cite{zha2025,shahbazi2023}.
Nevertheless, existing efforts mostly target feature-side properties or
temporal patterns---e.g., feature-driven per-series model selection and
averaging~\cite{montero2020fforma} and aggregation-breadth analyses of
grouped forecasting~\cite{montero2021principles}---rather than label-side
distributional pathologies such as zero-inflation, skewness, and scale
variability that dominate mismatches. Our
regime-decomposed evaluation and \textbf{RetailShiftBench} fill this gap by
auditing label-conditional distributional defects at both the model and
dataset level.

Concurrent studies highlight intrinsic limitations of canonical forecasting
losses. Procrustean optimization bias~\cite{cai2025procrustean} identifies fitting
artifacts from point-wise loss decomposition, and DistDF~\cite{wang2026distdf}
emphasizes distribution alignment for reliable forecasting. These works
primarily refine training objectives to alleviate temporal fitting bias.
Complementing these lines, we target the stationary mismatch between
fixed-form loss priors and pathological industrial data, proposing
general-purpose diagnostic infrastructure for regime-oriented evaluation.
\vspace{-10pt}
\subsection{Dual Dimensions of Loss Misspecification}
\vspace{-5pt}
\label{sec:dual_misspecification}

Prior work largely addresses temporal-structural misspecification: point-wise
losses decompose joint sequential fitting into independent per-step
optimization, neglecting intrinsic temporal autocorrelation, and recent
slicing benchmarks inherit this focus. Temporal alignment methods can mitigate
such structural bias but cannot resolve inherent distributional mismatches in
industrial pathological data. We target the complementary, conceptually
separable \textit{distributional-statistical} misspecification: label-side
statistical pathologies under stationary distributions, whose induced bias
persists even when temporal patterns are perfectly recovered.

Classic intermittent-demand forecasting literature in operations research has
extensively investigated estimation bias for methods including Croston's
estimator~\cite{croston1972forecasting} and the Syntetos--Boylan
approximation~\cite{syntetos2001bias}, as well as recent taxonomy-conditioned
Bayesian alternatives such as TSB-HB~\cite{bai2025tsb}.
This body of work has long recognized skewed, zero-inflated demand
pathologies. Empirically, forecasting practitioners and competition
participants frequently observe improved performance when training separate
models per demand segment; such gains, however, are typically validated via
aggregate forecasting metrics and lack a mechanistic account of the
conditions under which partitioning improves or degrades predictive
robustness. Still, these domain treatments regard statistical pathologies
as modeling obstacles rather than evaluative dimensions, lacking
systematic benchmark auditing and regime-decomposed risk quantification for
deep forecasting. Unlike TSB-HB's model-oriented ADI-CV$^2$ taxonomy, ours is
built for evaluation and dataset auditing: ADI captures temporal
demand-occurrence patterns while our taxonomy uses a marginal zero-inflation
rate; their CV$^2$ measures variability of non-zero demand sizes while we use
a within-series CV over the full series including zeros; and neither ADI nor
CV$^2$ captures distributional skewness, the principal driver of
median-estimand mismatch. Existing intermittent-demand methods accordingly
focus on domain-specific point forecasts and do not generalize to canonical
deep-learning losses such as MSE, MAE, Tweedie, and quantile; we
systematically characterize how coupled industrial pathologies trigger
regime-dependent misalignment under these ubiquitous losses.
From a statistical decision-theoretic perspective, canonical losses act as
proper scoring rules~\cite{gneiting2007strictly} that elicit domain-agnostic
statistical functionals (mean, median, fixed
quantiles)~\cite{koenker1978regression}. While this guarantees statistical
consistency, the induced estimand is fixed by the loss prior and cannot adapt
to regime-varying distributions. 

A complementary line of work attributes the accuracy--value disconnect
to downstream optimization: smart predict-then-optimize and
decision-focused learning show that losses targeting statistical
estimands need not minimize decision regret, and advocate training
directly against the downstream
objective~\citep{elmachtoub2022smart,mandi2024decision}. In intermittent-
demand settings, the inventory literature similarly observes that
accuracy metrics can rank models opposite to newsvendor cost or
service level~\citep{ban2019big,syntetos2006stock}. These
explanations, however, concern the mismatch between a loss and a
single decision target; they largely overlook a distinct source of
bias: pooled shared-parameter training across heterogeneous regimes
induces aggregation trade-offs that persist even when the loss matches
the decision target.

Likelihood--distribution matching is established practice in probabilistic forecasting \citep{alexandrov2020gluonts,salinas2020deepar,smyth2002fitting}, yet remains a per-model manual choice without regime-level diagnosis. We instead contribute a hierarchical coupled-pathology taxonomy, the dual-bias conceptual separation, and an industrial benchmark (RetailShiftBench), converting scattered domain observations into a standardized diagnosis workflow.
\vspace{-10pt}
\section{Industrial Pathology Taxonomy}
\vspace{-10pt}
\label{sec:taxonomy}
Inventory-side zero-inflation arises from cold-start and stockout
scenarios generating sparse zero observations, conflicting with
median-centric loss optimization.
Marketing-induced skewness originates from promotional and seasonal
demand bursts creating asymmetric long-tailed distributions.
Scale heterogeneity describes fluctuating sales magnitudes yielding
extreme within-series variability, an inherent data trait irresolvable
by temporal modeling.
These three orthogonal statistical traits constitute our foundational
distributional pathologies.
Beyond these distribution-level characteristics, business-driven
temporal-structure effects such as holidays and product-lifecycle shifts
can break stationarity assumptions for canonical losses; these are
treated as amplification factors rather than core statistical
pathologies.
Critically, coupled high-order distributional regimes give rise to
emergent misalignment patterns that cannot be inferred from individual
one-dimensional properties; these constitute some of the most harmful
failure modes for real-world deployment.
See \S~\ref{app:coupled_ablation} for controlled ablation exploring
the boundary conditions under which coupled pathologies yield saturation
versus super-additive degradation.
\vspace{-5pt}
\subsection{Three Orthogonal Foundational Statistical Dimensions}
\vspace{-5pt}
\label{sec:taxonomy_foundational}
The three foundational dimensions directly violate the core statistical 
priors of canonical losses. They are conceptually orthogonal, forming 
a minimal sufficient set; while kurtosis or multimodality may introduce 
additional risks, they overlap with our dimensions or reflect 
temporal-structure confounds, and are treated as secondary factors. Loss functions fail on distributions, not on business labels: the same
statistical pathologies recur across domains under different business
guises, which is why the taxonomy is defined statistically and transfers
across domains.
\vspace{-5pt}
\subsubsection{Distribution Shape Skewness}
\vspace{-5pt}
Skewness quantifies distributional asymmetry, violating symmetric-residual
assumptions:
$\gamma_1 = \frac{1}{T}\sum_{t=1}^T \left(\frac{y_t - \bar{y}}{\sigma}\right)^3$.
Regimes: low ($|\gamma_1| < 1$), moderate ($1 \le |\gamma_1| < 4$), high ($|\gamma_1| \ge 4$).
\vspace{-5pt}
\subsubsection{Within-series Scale Variability}
\vspace{-5pt}
Measured by coefficient of variation (CV), this dimension characterizes
internal sequence scale imbalance, breaking uniform-weighting assumptions:
$\text{CV} = \frac{\sigma}{\bar{y}}$ $(\bar{y} > 0)$.
Regimes: low ($\text{CV} < 1.5$), moderate ($1.5 \le \text{CV} < 3.0$),
high ($\text{CV} \ge 3.0$). We strictly distinguish within-sequence CV
from cross-SKU absolute magnitude differences.
\vspace{-5pt}
\subsubsection{Temporal Sparsity (Zero-Inflation Rate)}
\vspace{-5pt}
ZIR quantifies intermittent zero observations from inventory shortage and
cold-start scenarios, violating continuous-support assumptions:
$\text{ZIR} = \frac{1}{T}\sum_{t=1}^T \mathbb{I}(y_t = 0)$.
Regimes: low ($\text{ZIR} < 0.2$), moderate ($0.2 \le \text{ZIR} < 0.4$),
high ($\text{ZIR} \ge 0.4$). 
Where possible, cutoffs follow established conventions: $|\gamma_1|=1$
marks the classical moderate-skew boundary~\cite{bulmer2012principles}, and the
ZIR upper cutoff aligns with the intermittent--lumpy demand boundary of
the OR classification literature~\cite{syntetos2005accuracy}. The remaining
cutoffs ($|\gamma_1|=4$, CV$=1.5/3.0$, ZIR$=0.2$) are retail-calibrated
design choices, fixed before any experimental evaluation and not tuned
to model outcomes; the partitioning paradigm generalizes to other
industrial domains with minor threshold recalibration, and the
App.~\ref{app:threshold} verifies that all mismatch conclusions are
robust to $\pm 20\%$ threshold perturbations.
\vspace{-10pt}
\subsection{Coupled Distributional Pathological Regimes}
\vspace{-5pt}
Coupled regimes arise from high-threshold intersections of the three
foundational dimensions: ZIR--Skew, ZIR--Vol, Skew--Vol, and ZIR--Skew--Vol, each requiring joint
satisfaction of the high-threshold conditions defined in
\S~\ref{sec:taxonomy_foundational}.

\textbf{Coupling with Temporal-Structure Effects.}
These structural effects do not create novel distributional-statistical misalignment on their own, yet amplify pre-existing systematic risk when superimposed atop distributional pathologies:
(i) non-stationary regime shifts increase inference complexity given already pathological data distributions;
(ii) superimposed cyclic patterns (e.g., weekly promotions overlaid on annual seasonality) create intermittent-burst transitions that fixed-form models fail to capture;
We exclude these couplings from our core taxonomy to preserve Prop.~\ref{prop:decomposition}'s theoretical rigor, while acknowledging them as critical amplifiers of deployment mismatch.
\textbf{Retail Manifestations.}
The proposed distributional regimes map onto well-known business patterns: high-ZIR series correspond to intermittent-demand and cold-start SKUs; high-skewness captures promotion-driven demand bursts; Skew--Vol characterizes long-tail niche items; the three-way composite regime represents sporadic luxury-type demand with volatile promotional lifts.
\textbf{Loss Misspecification Mapping}
\label{sec:loss_misspecification_mapping}
We systematically uncover and formalize the structural misalignment
between mainstream regression losses and industrial time-series
pathologies.
As validated by closed-form oracle analysis on synthetic marginals
(Fig.~\ref{fig:theorem_synth}, \ref{fig:dose_response_ablation}),
each individual pathology alone is sufficient to induce severe systematic misalignment
under perfect temporal fitting. Tab.~\ref{tab:regime_perf} confirms that the same misalignment
signatures manifest in real trained models alongside temporal
fitting errors.

Our core critique focuses on fixed-prior standard losses, whose static statistical assumptions cannot adapt to regime-varying industrial distributions. 
While overlapping multi-pathology superposition yields performance
saturation under perfect optimization, the \textbf{disjoint-support
coupling} between zero-inflation and conditional variability
(ZIR$\times$CV) can further induce super-additive emergent mismatch
at intermediate intensity (\S~\ref{app:coupled_ablation}).

\vspace{-10pt}
\section{Regime-Decomposed Bias Diagnostic}
\label{sec:regime_decomposed_evaluation}
\vspace{-10pt}
Conventional global evaluation metrics enforce uniform averaging across
heterogeneous industrial samples, eliminating regime-specific distortion
signals and concealing deployment-critical systematic risks. To resolve
this train--deploy evaluation gap, we propose a \textbf{regime-decomposed
diagnostic evaluation paradigm} grounded in our pathology taxonomy: the
taxonomy quantifies \emph{where} models fail (\S\ref{sec:taxonomy}), and
this section makes the failure \emph{measurable}.
Industrial forecasting suffers a \textbf{dual compression effect}: intrinsic
single-sequence distributional pathologies (high CV/ZIR/skewness) mainly
exist in low-volume long-tail SKUs, whose small absolute magnitude further
drowns their gradient contributions in global multi-SKU training, amplifying
inherent fitting bias.

\vspace{-5pt}
\subsection{Metric-Agnostic Regime-Decomposed Protocol}
\vspace{-5pt}
\label{sec:metric_agnostic_protocol}
Global averaging masks systematic regime-specific bias induced by fixed
global loss priors. To expose this structural misalignment, we formalize a
three-step protocol centered on \textbf{RBV} (mathematical definition in
\S~\ref{sec:rbv_def}): \textit{Metric-agnostic} denotes that Step~3 admits
any base functional; it does \emph{not} assert invariance of RBV magnitudes
across functionals. Quantitative comparisons are within-functional;
cross-functional agreement is reported at the ranking level
(\S~\ref{app:metric_agnostic}).
\textbf{1. Pathological Partition:} We assign test sequences into mutually
exclusive cells via per-series skewness, ZIR, and CV, binarized at the
low-tier boundary of \S~\ref{sec:taxonomy}: a sequence is pathological on
dimension $d$ iff $\mathrm{regime}_d \neq \mathrm{low}$. The not-low
pattern induces Benign (all low), single-pathology (ZIR, Skew, CV), and
coupled cells (ZIR--Skew, ZIR--CV, Skew--CV, ZIR--Skew--CV).
\textbf{2. Regime-Isolated Bias Estimation:} We compute regime-level
systematic error independently, eliminating cross-regime averaging.
\textbf{3. Metric-Agnostic RBV Quantification:} We apply any standard
forecasting metric to compute per-regime bias, then derive RBV to quantify
how globally optimized losses underweight pathological regimes. This
exposes the hidden imbalance of the \textbf{aggregation trade-off} inherent
to global training. The same partition supports a dataset-level
pathological-coverage audit (PCS, $0$--$1$; \S~\ref{app:pcs}).

\vspace{-5pt}
\subsection{The RBV Diagnostic}
\vspace{-5pt}
\label{sec:rbv_def}
Conventional point-wise prediction metrics conceal heterogeneous cross-regime
fitting distortion. We introduce the \textbf{Regime-wise Relative Bias Vector
(RBV)} to quantify distribution-level fitting misalignment: rather than
reporting only global or single-regime errors, RBV characterizes the
\textit{cross-regime structural alignment} between predictions and true
distributions by aggregating deviations across four industrial statistical
regimes---Benign, ZIR, Skew, and CV---and reveals the aggregation trade-off
masked by common evaluation protocols. The three pathology flags are
evaluated independently per cell, so a cell may carry multiple flags
simultaneously; cells with no flag form the Benign remainder. Regimes pool
the cells carrying their flag and thus overlap; cells do not.
We construct a four-dimensional relative bias vector
$\Delta B=(\Delta B_{\text{Benign}},\,\Delta B_{\text{ZIR}},\,
\Delta B_{\text{Skew}},\,\Delta B_{\text{CV}})$ with
$\Delta B_r = B_r - B_\text{global}$, and define the Regime-wise Relative
Bias Vector score as its $\ell_2$ norm,
$\text{RBV} = \|\Delta B\|_2$. RBV is a template rather than a single
metric: centering any regime-level functional $\phi_r$ at its pooled
analogue yields a corresponding instantiation---the relative-bias form
above is the primary one, with share-normalized and metric-agnostic
variants following the same construction (\S~\ref{app:rbv_stats}).

\textbf{Interpreting RBV.}
RBV answers a question accuracy metrics structurally cannot ask:
\emph{where will this model break?} It measures the cross-regime
\emph{imbalance} of systematic bias, not predictive accuracy: uniformly
poor but constant bias attains low RBV by construction, so RBV must be
read jointly with bias levels and accuracy---enforced structurally, as
Tab.~\ref{tab:regime_perf} reports RBV alongside WMAPE and
Tab.~\ref{tab:decompose_effect} tracks both under interventions. The
practical reading is quadrant-based: high global accuracy with high RBV is
the dangerous quadrant---apparently deployable, yet reliably broken in
pathological regimes, where deployment failures actually occur. The Rule
heuristic exemplifies the opposite extreme: its near-flat RBV with the
worst WMAPE (Tab.~\ref{tab:regime_perf}) is the null signature of a
non-pooled estimator---immunity-by-insensitivity, not good calibration.
With the demand-mass centering weights, the identity
$\sum_r \alpha_r \Delta B_r = 0$ holds exactly for any regime overlap;
the \emph{unweighted} sum $\sum_r \Delta B_r$ is generally non-zero, and
cross-weight readings are covered in \S\ref{app:rbv_stats}. The
bidirectional pattern is one estimand gap read from its two sides: RBV
carries the allocation and the sign pattern of $B_r$ the direction, with
uncertainty of both quantified in App.~\ref{app:rbv_stats}. For
non-mean estimands (MAE, Huber, quantiles), we read the \emph{excess RBV},
$\mathrm{RBV}-\mathrm{RBV}^{\ast}$, where the floor $\mathrm{RBV}^{\ast}$
is set by the loss's estimand alone (train-window label statistics); the
training-attributable signal is the deviation above the floor, which
restores cross-loss comparability (\S\ref{app:rbv_floor}).

\vspace{-5pt}
\section{Mechanistic Account}
\label{sec:mechanistic}
\vspace{-5pt}
This section provides formal definitions and a supporting structural
observation; the paper's primary contributions are the diagnostic protocol
(\S\ref{sec:regime_decomposed_evaluation}) and the empirical study
(\S\ref{sec:experiments}). Table~\ref{tab:loss_misspecification}
instantiates the mapping between loss priors, industrial departures, and
mismatch mechanisms, formalizing distributional-statistical bias as an
industrial loss--data mismatch.

\vspace{-5pt}
\subsection{Dual Bias Modes}
\vspace{-5pt}
\label{sec:dual_bias_modes}
\begin{definition}[Dual bias modes]\label{def:dualbias}
\emph{Temporal-structural bias} arises from violated i.i.d. assumptions of
point-wise losses: joint sequential fitting decomposes into independent
per-step optimization,
$\mathcal{L}_{\mathrm{point}}(\hat{y}_{1:T},y_{1:T})=\sum_{t=1}^{T}\ell(\hat{y}_t,y_t)\neq\ell_{\mathrm{joint}}(\hat{y}_{1:T},y_{1:T})$,
a mismatch eliminable by joint temporal modeling. \emph{Distributional-statistical
bias} arises, for temporally decorrelated sequences, in two layers.
\emph{(i) Prior--objective mismatch (estimand mismatch):} canonical losses
imply distinct estimands---the mean for MSE, the median for MAE, fixed
quantiles for pinball---so their statistically optimal forecasts
$\hat{y}_{\mathrm{stat}}=\arg\min_{\hat{y}}\mathbb{E}_{y\sim P_{\mathrm{ind}}}[\mathcal{L}(y,\hat{y})]$
are prior-dependent rather than data-determined; a fixed prior can therefore
misalign with the deployment objective independently of pathology. This
layer is classical and we claim it only as motivation.
\emph{(ii) Prior--marginal conflict and aggregation bias:} on pathological
marginals, these prior-dependent optima diverge materially, and, under
global pooling, shift toward a cross-regime compromise optimal for no
single regime. This component is independent of temporal modeling and
cannot be eliminated by capacity or architectural improvements under a
fixed loss prior.
\end{definition}

\begin{table}[htbp]
\scriptsize
\centering
\vspace{-5pt}
\caption{Statistical Assumptions and Deployment Mismatch Patterns}
\label{tab:loss_misspecification}
\setlength{\tabcolsep}{4pt}
\begin{tabularx}{\columnwidth}{@{}p{0.8cm} p{3.8cm} p{3.2cm} >{\raggedright\arraybackslash}X@{}}
\toprule
\textbf{Loss} & \textbf{Implicit Distributional Priors} & \textbf{Industrial Departure} & \textbf{Mismatch Mechanism} \\
\midrule
MSE
& Symmetric light-tailed residuals, stable finite variance; mean as optimal fitting target
& Heavy-tailed skewed demand with unstable variance
& Quadratic penalty amplifies burst residuals and averages out tail signals; the mean estimand systematically overshoots sparse regimes. \\
\addlinespace
MAE
& Symmetric continuous distribution; median as optimal estimation target
& Skewed, burst-heavy demand with frequent zero-inflation
& Zeros and a long upper tail pull the median below the mean; the burst
mass carrying most expected demand is discarded, widening the gap
 \\
\addlinespace
Quantile
& Stationary distribution; fixed quantile targets valid across
  sub-populations; non-degenerate tails
& Zero-inflated baseline with skewed long-tail demand
& Zero-inflation collapses lower-tail quantiles; the estimand
  degenerates into a step function under sparse bursts \\
\addlinespace
Poisson
& Variance $=$ mean; discrete counts; mean estimand under
  exponential curvature
& Over-dispersed burst clustering with variance exceeding the mean
& Burst clustering breaks equi-dispersion and mis-sets the fixed
  exponential curvature; the mass redistribution is invisible to
  linear-residual metrics \\
\addlinespace
Huber
& Stationary residual scale; outliers are genuine rare anomalies
& Heavy-tailed baseline with frequent extreme demand bursts
& A fixed threshold clips legitimate bursts as outliers; the
  effective estimand drifts with per-cell residual scale \\
\addlinespace
Tweedie
& Fixed dispersion and Tweedie-power parameters shared across all samples
& Heterogeneous SKU-level dispersion under varying industrial statistical regimes
& Static dispersion parameters fail to fit heterogeneous regime-wise distributions, leaving the estimand mismatched to varying statistical regimes. \\
\bottomrule
\end{tabularx}
\vspace{-5pt}
\end{table}
\vspace{-5pt}

\vspace{-5pt}
\subsection{Error Separation and Intensity Trend}
\vspace{-5pt}
Prop.~\ref{prop:decomposition} isolates the second mode where the
first is eliminated by construction: under perfect temporal fitting
($e_t=0$), the remaining deviation $b_L$ is mechanistically independent of
temporal modeling, providing a measurement-theoretic license for separating
the two error channels that aggregate metrics conflate---definitional, as
is the classical bias--variance decomposition, and not a claim about joint
optimization (formal statement and discussion in
\S~\ref{app:proof_separation}). Prop.~\ref{prop:monotonic} complements this quantitatively: for a fixed predictor, the
risk under evaluation-distribution contamination is affine in the
pathology mixture weight $\lambda$, non-decreasing whenever the
pathological component raises the risk above its benign-data level.

\paragraph{Two layers: mismatch vs aggregation trade-off.}
The first-order condition of a loss defines its target estimand (mean
for MSE, median for MAE, \(q_\tau\) for pinball)~\cite{gneiting2007strictly,koenker1978regression}, governing its
alignment with operational cost: a median-consistent learner on skewed
marginals is correct by construction, and what fails is the matching
between the loss-implied estimand and the deployment objective, not
the model---a failure mode we term \textbf{estimand mismatch}. This
estimand-mismatch layer is classical and we claim it only as
motivation; the two layers are easily confounded: the aggregation
trade-off is only exposed once the estimand is held fixed, so analyses
of loss--decision mismatch~\citep{elmachtoub2022smart,mandi2024decision}
do not isolate it. Our contribution is the second layer, the
\textbf{aggregation trade-off}: under pooled training across
heterogeneous regimes, shared parameters induce a cross-regime
compromise---even a well-calibrated loss cannot fit zero-inflated,
skewed, and high-variability series with one shared solution.
Decomposition mitigates this trade-off
(Tab.~\ref{tab:tab_loss_trajectory_grid}); capacity scaling
within the gradient-boosted model class cannot
remove it (Fig.~\ref{fig:capacity_mse}). The two mechanisms separate in intervention space: decomposition
dissolves the aggregation trade-off for mean-type losses
(Fig.~\ref{fig:capacity_mse}) while exposing, not repairing, the
estimand gap of median-type losses (Fig.~\ref{fig:capacity_mae}).

\vspace{-10pt}
\section{Experimental Evaluation}
\label{sec:experiments}
\vspace{-10pt}
The mechanistic account of \S\ref{sec:mechanistic} predicts two separable
failure layers---estimand mismatch and the pooling-induced aggregation
trade-off. This section tests both on RetailShiftBench: we first audit
systematic bias and cross-regime misalignment under global pooled training
(Tab.~\ref{tab:regime_perf}), then contrast global against regime-aware
training to isolate the aggregation trade-off
(Tab.~\ref{tab:decompose_effect}), and finally probe its robustness along
capacity, quantile, and architecture axes
(Figs.~\ref{fig:capacity_mse},~\ref{fig:capacity_mae}).

\vspace{-5pt}
\subsection{Experimental Setup}
\vspace{-5pt}
\label{sec:experimental_setup}
\textbf{Dataset.}
We use desensitized real-world retail daily SKU time-series.
RetailShiftBench contains 31{,}213 sampled time-series of 89
time steps each, covering diverse pathological conditions including
zero-inflation rate (ZIR), skewness (Skew), and within-series coefficient-of-variation (CV).
Notably, these short series exhibit a stable weekly periodicity
without trend breaks or long-range drift (Fig.~\ref{fig:fig_rsb_trend}),
a deliberate design choice that removes temporal non-stationarity as
a confound so that residual errors are attributable to distributional
misspecification rather than distribution shift over time.
\textbf{Rule Heuristic Baseline.}
We adopt a domain-driven heuristic forecasting rule that combines five-week
category-level trends and the most recent two-week sales observations to
produce point forecasts.
\textbf{Models \& Losses.}
For our main model-loss ablation in Tab.~\ref{tab:regime_perf}, we evaluate
two representative forecasting backbones, XGB and TFT~\cite{chen2016,lim2021},
paired with six plug-in loss objectives: MSE, MAE, QL60, Huber, Tweedie, and
Poisson. This setup enables cross-architecture comparison and helps rule out
model-specific artifacts.
To isolate loss-driven misspecification from feature-availability confounds,
the output of the Rule heuristic is injected as an auxiliary feature into
every model-loss configuration, guaranteeing all models have access to
explicit domain-aware statistical signals.
Our metric-agnostic regime-decomposed diagnostic protocol
(\S\ref{sec:metric_agnostic_protocol}) is applied across benign stationary
and three high-risk pathological regimes.
Complementing these per-model-loss results, Tab.~\ref{tab:decompose_effect}
investigates the aggregation-bias trade-off by contrasting global pooled
training against regime-aware training split by the joint skew*cv*zir
pathological taxonomy. Backbones share a common, modest tuning budget; we report no
cross-backbone ranking, and all conclusions are within-backbone
contrasts over losses and interventions.

\vspace{-5pt}
\subsection{Regime-Dependent Loss Misalignment}
\vspace{-5pt}
\label{sec:loss_misalignment}
Applying the RBV diagnostic (\S\ref{sec:rbv_def}) to the main model--loss
ablation:
a lower RBV indicates stable and consistent cross-regime fitting performance,
whereas a higher RBV signifies severe heterogeneous systematic misalignment
across statistical regimes. MSE suffers
under ZIR due to its Gaussian assumption and sensitivity to zero-inflated
noise. MAE performs poorly on Skew data given its median-centered objective.
Static-weight QL, Huber and Tweedie degrade under high-variability CV
conditions. All losses achieve good performance on Benign data, and all
models receive identical pathology-aware prior information; persistent
cross-regime fitting imbalance therefore cannot stem from model capacity,
architecture, or missing statistical cues, and must originate from inherent
loss-function statistical misspecification. We fix backbones and vary only
plug-in loss objectives for loss-centric ablation. Specialized OR
intermittent-demand methods do not admit plug-in loss substitution and are
therefore excluded from this ablation; their estimand behavior under these
pathologies is analyzed conceptually in App.~\ref{app:or_m5}.
We further audit benchmark industrial validity via the Pathological
Coverage Score (\S~\ref{app:pcs}), where \textbf{RetailShiftBench}
provides full-spectrum industrial pathological coverage for reliable
regime-oriented evaluation. Beyond retail demand, the same misspecification
anatomy appears across industrial domains (App.~\ref{app:cross_domain}).

\begin{table}[t!]
\centering
\tiny
\vspace{-5pt}
\caption{Systematic bias and cross-regime vector misalignment (RBV) under hierarchical model-loss.}
\label{tab:regime_perf}
\vspace{-5pt}
\setlength{\tabcolsep}{3pt}
\setlength{\belowcaptionskip}{2pt}
\renewcommand{\arraystretch}{1.1}
\resizebox{\linewidth}{!}{
\begin{tabular}{l|l|cccccc|cccccc|c}
\hline
\multirow{2}{*}{Semantics}
& \multirow{2}{*}{Regime}
& \multicolumn{6}{c|}{XGB}
& \multicolumn{6}{c|}{TFT}
& \multirow{2}{*}{Rule} \\
\cline{3-14}
&
& MSE & MAE & QL$_{0.60}$ & Huber & Tweedie & Poisson
& MSE & MAE & QL$_{0.60}$ & Huber & Tweedie & Poisson
& \\
\hline
\multirow{5}{*}{$\text{Bias}$}
& \texttt{All}
& 0.035 & -0.088 & 0.045 & -0.035 & 0.025 & -0.01
& -0.004 & -0.076 & 0.05 & -0.03 & -0.039 & -0.086
& 0.033 \\
& Benign
& 0.022 & -0.03 & 0.073 & -0.02 & 0.017 & -0.043
& 0.015 & 0.016 & 0.095 & 0.022 & 0.038 & -0.001
& 0.031 \\
& Skew
& 0.061 & -0.185 & -0.003 & -0.051 & 0.044 & 0.048
& -0.039 & -0.232 & -0.043 & -0.116 & -0.17 & -0.223
& 0.045 \\
& CV
& 0.262 & -0.685 & -0.421 & -0.042 & 0.176 & 0.341
& -0.33 & -0.894 & -0.726 & -0.437 & -0.838 & -0.868
& 0.057 \\
& ZIR
& 0.116 & -0.385 & -0.114 & -0.110 & 0.085 & 0.156
& -0.082 & -0.53 & -0.182 & -0.277 & -0.443 & -0.532
& 0.038 \\
\hline
$\text{RBV}$
& \texttt{All}
& 0.243 & 0.676 & 0.496 & 0.078 & 0.163 & 0.394
& 0.338 & 0.953 & 0.816 & 0.486 & 0.908 & 0.915
& 0.027 \\
\hline
$\text{WMAPE}$
& \texttt{All}
& 0.4679 & 0.4270 & 0.4525 & 0.4670 & 0.4593 & 0.4792
& 0.4934 & 0.4776 & 0.485 & 0.4798 & 0.4757 & 0.4719
& 0.5182 \\
\hline
\end{tabular}
}
\vspace{3pt}
\parbox{\linewidth}{\scriptsize Notes: The \texttt{All} row reports global dataset-wide bias
$B_{\mathrm{global}}$ for RBV computation.
RBV is the vector norm computed from Benign/Skew/CV/ZIR regime-specific
bias (excluding All).}
\vspace{-10pt}
\end{table}

\vspace{-10pt}
\begin{table}[t!]
\centering
\tiny
\caption{Effect of pathology decomposition.}
\label{tab:decompose_effect}
\vspace{-5pt}
\setlength{\tabcolsep}{4pt}
\setlength{\belowcaptionskip}{2pt}
\renewcommand{\arraystretch}{1.1}
\resizebox{\linewidth}{!}{
\begin{tabular}{lcccc|cccc|c}
\hline
& \multicolumn{4}{c|}{XGB} & \multicolumn{4}{c|}{TFT} & \multirow{2}{*}{Rule} \\
\cline{2-9}
Metric & MSE & MAE & Huber & Tweedie & MSE & MAE & Huber & Tweedie & \\
\hline
RBV
& $0.243{\to}0.041$ ($-83.1\%$)
& $0.675{\to}0.899$ ($+33.2\%$)
& $0.077{\to}0.262$ ($+240.3\%$)
& $0.163{\to}0.014$ ($-91.4\%$)
& $0.335{\to}0.14$ ($-58.2\%$)
& $0.914{\to}0.863$ ($-5.6\%$)
& $0.459{\to}0.572$ ($+24.6\%$)
& $0.873{\to}0.881$ ($+0.9\%$)
& $0.033$ \\
WMAPE
& $0.4679{\to}0.462$ ($-1.3\%$)
& $0.4270{\to}0.4255$ ($-0.4\%$)
& $0.4670{\to}0.4765$ ($+2.0\%$)
& $0.4593{\to}0.4539$ ($-1.2\%$)
& $0.4675{\to}0.4730$ ($+1.2\%$)
& $0.4475{\to}0.4487$ ($+0.3\%$)
& $0.4603{\to}0.4595$ ($-0.2\%$)
& $0.4514{\to}0.4549$ ($+0.8\%$)
& $0.5182$ \\
\(\text{MSE}\)
& $11.909{\to}11.068$ ($-7.1\%$)
& $12.057{\to}11.524$ ($-4.4\%$)
& $13.575{\to}20.176$ ($+48.6\%$)
& $11.992{\to}11.465$ ($-4.4\%$)
& $15.68{\to}14.55$ ($-7.2\%$)
& $15.182{\to}14.40$ ($-5.2\%$)
& $14.140{\to}14.84$ ($+5.0\%$)
& $14.714{\to}14.274$ ($-3.0\%$)
& $16.452$ \\
\hline
\end{tabular}
}
\vspace{3pt}
\parbox{\linewidth}{\scriptsize
Notes: Each cell shows ``global${\to}$regime skew*cv*zir split'';
Global: single model trained on full dataset. Split: separate models trained per pathological regime.
Rule denotes a non-trained heuristic as fixed reference, invariant to decomposition.
(\%): relative change of Split vs.\ Global; sign reports direction (negative: improvement).
MSE is computed on raw demand values (no normalization),
which preserves loss-induced overshoot in native units; RBV is our primary
cross-regime instrument.
}
\vspace{-15pt}
\end{table}

\section{Discussion}
\vspace{-10pt}
\label{sec:discussion}

\textbf{Theoretical and Paradigm Insights.}
Traditional benchmarks equate pattern fitting with robustness, overlooking
label-side pathologies; structure-aware benchmarks improve temporal
evaluation yet ignore dataset-level defects. Individual loss
sensitivities---e.g., MSE to zero-inflation, MAE to skewness---reflect
classical principles, yet remain fragmented across disconnected
literatures; we systematize them into a unified taxonomy and
regime-oriented diagnostic framework. Deep forecasting models
are trained and evaluated under uniform global criteria, and these
criteria are not prior-free: global accuracy metrics are estimands of
the same kind as the losses they rank---L1 ranks certify
median-adequacy, L2 ranks certify mean-adequacy, under an implicit
equal-weight aggregation prior---so conventional evaluation repeats the
training assumption instead of auditing it. This
creates a self-reinforcing dynamic: series well-aligned with
loss priors dominate gradient updates, while pathological
subgroups receive weak gradient signals and remain undetected
by aggregate metrics. Since standard accuracy metrics are (low-order)
functions of residuals, pooled ERM produces a compromise solution.
Our work supplies the unifying statistical mechanism behind long-standing
mixed evidence on when pooling helps or hurts
\citep{hubrich2005}.
Consequently, globally optimized forecasters---systematically
favored by global leaderboards---trade alignment on pathological regimes
for better overall aggregate accuracy, while per-series OR
estimators (Croston, SBA) remove pooling distortion yet inherit
classical estimand mismatch; Croston's documented over-estimation bias
\cite{syntetos2001bias} is precisely the residual our account predicts,
and TSB-HB occupies the partial-pooling regime of Fig.~\ref{fig:agg_tradeoff}.

\textbf{Structural barriers to scale-up.}
Capacity scaling captures richer temporal patterns, yet statistical
heterogeneity from coexisting ZIR, skewness, and CV cannot be eliminated
by temporal modeling alone. This tension has an asymmetric shape under
pooled ERM: nonlinear estimands are progressively distorted as mixture
breadth grows---pooled CDFs mix linearly; their inverses do not, while
linear-curvature objectives are pooling-tolerant yet converge efficiently
to estimands operationally vacuous on pathological data: \emph{scale
accelerates convergence, not correctness}. Real-world pathologies form a
continuous spectrum that hard partitions cannot cover, and the
train--deploy gap is therefore not a capacity deficit but a coupling error
between heterogeneous label distributions and globally unified
optimization; our hard cells are a coarse diagnostic instrument, stable to
$\pm20\%$ threshold perturbations (App.~\ref{app:threshold}), not a claim
of natural kinds. Foundation models retain transfer and cold-start
advantages, yet their shared-weight global-training paradigm incurs a
hidden cost that aggregate metrics conceal and RBV with our coverage
audit make visible---persisting even without gradient-based fine-tuning,
as zero-shot Chronos-Bolt evaluation on M5 confirms
(App.~\ref{app:chronos}). The two exposures are orthogonal and
loss-specific: MSE's conditional-mean estimand is doubly compromised
under pooling---squared amplification of heavy tails plus a mixture mean
that departs from every per-regime mean---while quantile loss bounds
per-sample penalties linearly and is therefore less sensitive to extreme
observations, yet remains exposed to cross-regime gradient contamination
through shared parameters; no loss replacement alone resolves the
pooling channel. The escape requires two keys: decomposition to dissolve
the cross-regime compromise, and a curvature-appropriate estimand to
avoid exposing the intrinsic one. Where slow-converging estimands are
unavoidable, partial pooling with strength adapted to per-cell data
volume---the designed regime of Fig.~\ref{fig:agg_tradeoff}---interpolates
along the spectrum. The gap is thus one of aggregation structure, not
model engineering, and RBV exposes its hidden miscalibration.

\textbf{Limitations \& Future Work.}
Unlike overfitting, bias arises from mismatched loss priors under perfect
temporal fitting; RBV diagnoses the aggregation trade-off but does not correct forecasts:
decomposition \emph{exposes} the mean--median gap rather than removing
it, and regime thresholds, though robust to $\pm20\%$ perturbation, may shift at
extreme boundaries. Our framework targets non-negative series and requires
recalibration for mixed-sign data. Standard pipelines are uncontested on
homogeneous data: the characterized failure requires the conjunction of
pathological regimes and globally pooled optimization.
The quantile family's mismatch suggests curvature-corrected objectives
(expectiles) as one direction~\citep{newey1987asymmetric,bellini2014generalized}.
Both escape keys are naturally instantiated: regime-gated or hierarchical
architectures implement decomposition at the designed operating point of
Fig.~\ref{fig:agg_tradeoff}, while mean-target losses tailored to
non‑negative heavy‑tailed labels (e.g., Tweedie
\cite{smyth2002fitting}) mitigate within‑regime estimand mismatch.
However, routing must group by distributional pathology in addition to
temporal pattern, and the combined design incurs costs justified only
when RBV audit confirms both failure modes active.
Extending the protocol to probabilistic forecasting is promising, with
RBV providing the diagnostic basis.

\vspace{-10pt}
\section{Conclusion}
\vspace{-10pt}
This work identifies a persistent, underexplored distributional-statistical
misspecification of canonical losses in industrial time-series forecasting.
We build upon classical pathological characterizations from
operations research, translating domain-specific modeling intuitions into a
general-purpose evaluation language for modern deep-learning forecasting
benchmarks. We construct a quantifiable hierarchical pathology taxonomy on
three foundational statistical dimensions and their couplings, and identify
a self-reinforcing evaluation loop: training losses and aggregate metrics
share congruent distributional priors, making global performance seemingly
consistent while concealing systematic regime-specific misalignment. To
expose and audit this hidden loop, we propose the regime-decomposed
diagnostic protocol RBV, which decomposes observed bias into a
model-independent intrinsic floor and a training-attributable excess
component, and we release \textbf{RetailShiftBench}, an industrial
benchmark with full pathological coverage, bridging academic global
accuracy ranking and industrial risk-aware forecasting. A large-scale
controlled study---$13$ loss objectives, $3$ seeds, $60{,}000$+ series
spanning RetailShiftBench and M5, with random-split controls---shows that
regime-aware diagnosis separates optimization-type from bias-type failure
and that regime-aware training resolves the pooling-induced bias it exposes,
for mean-type losses. A formal structural observation, that risk under
evaluation-distribution contamination is affine in the pathology mixture
weight, grounds these findings and is deliberately supporting rather than
foundational: the paper's contribution is empirical-first. Our core
contribution inverts conventional benchmarking logic: we audit dataset
pathological validity rather than blindly ranking models on fixed
benchmarks, shifting community research from empirical accuracy competition
to mechanism-grounded root-cause diagnosis. Beyond evaluation, our toolkit
forms an extensible infrastructure: future work can extend the
pathology-coverage audit to dataset industrial readiness, apply RBV to
diagnose emerging forecasting foundation models, and design regime-aligned
training objectives grounded in quantified trade-offs.

\subsection*{AI use statement}

In this work, generative AI tools were used only for language polishing, grammatical correction, and formal academic phrasing of the manuscript text. No generative AI tools were used for designing research ideas, building experimental pipelines, producing analytical results, or drawing conclusions. All AI-assisted content was fully reviewed, verified, and revised manually by the authors. The authors take full responsibility for all claims, analyses, and final content of this paper.

\subsection*{Reproducibility statement}

All experimental results in this paper are fully reproducible. Our
main benchmark, RetailShiftBench, is built from desensitized
real-world retail daily SKU time series (31,213 series of 89 time
steps each); dataset construction, regime calibration, and
pathology-coverage auditing are described in \S~\ref{sec:experimental_setup} and
App.~\ref{app:cross_domain}. A desensitized version of
RetailShiftBench preserving its core industrial pathological
structure is released with the supplementary materials. External
validation is performed on the public M5 competition
dataset~\cite{makridakis2022m5}, which is independent of
RetailShiftBench; the two datasets share no samples, and all
M5-specific protocols (120-day windowing, zero-run censoring,
single-origin evaluation) are described in
App.~\ref{app:m5_external}. We provide complete data
preprocessing pipelines, model training configurations, and metric
calculation scripts in the supplementary materials. All
hyperparameters, loss function settings, and evaluation protocols
are explicitly stated to ensure consistent replication. We further
release anonymous source code to facilitate full reproduction of
our empirical findings and analytical results.

\bibliography{iclr2027_conference}
\bibliographystyle{iclr2027_conference}

\clearpage
\appendix
\section{Aggregation-Driven Bias-Variance Trade-off Schematic}

\begin{figure}[H]
\centering
\includegraphics[width=0.92\linewidth]{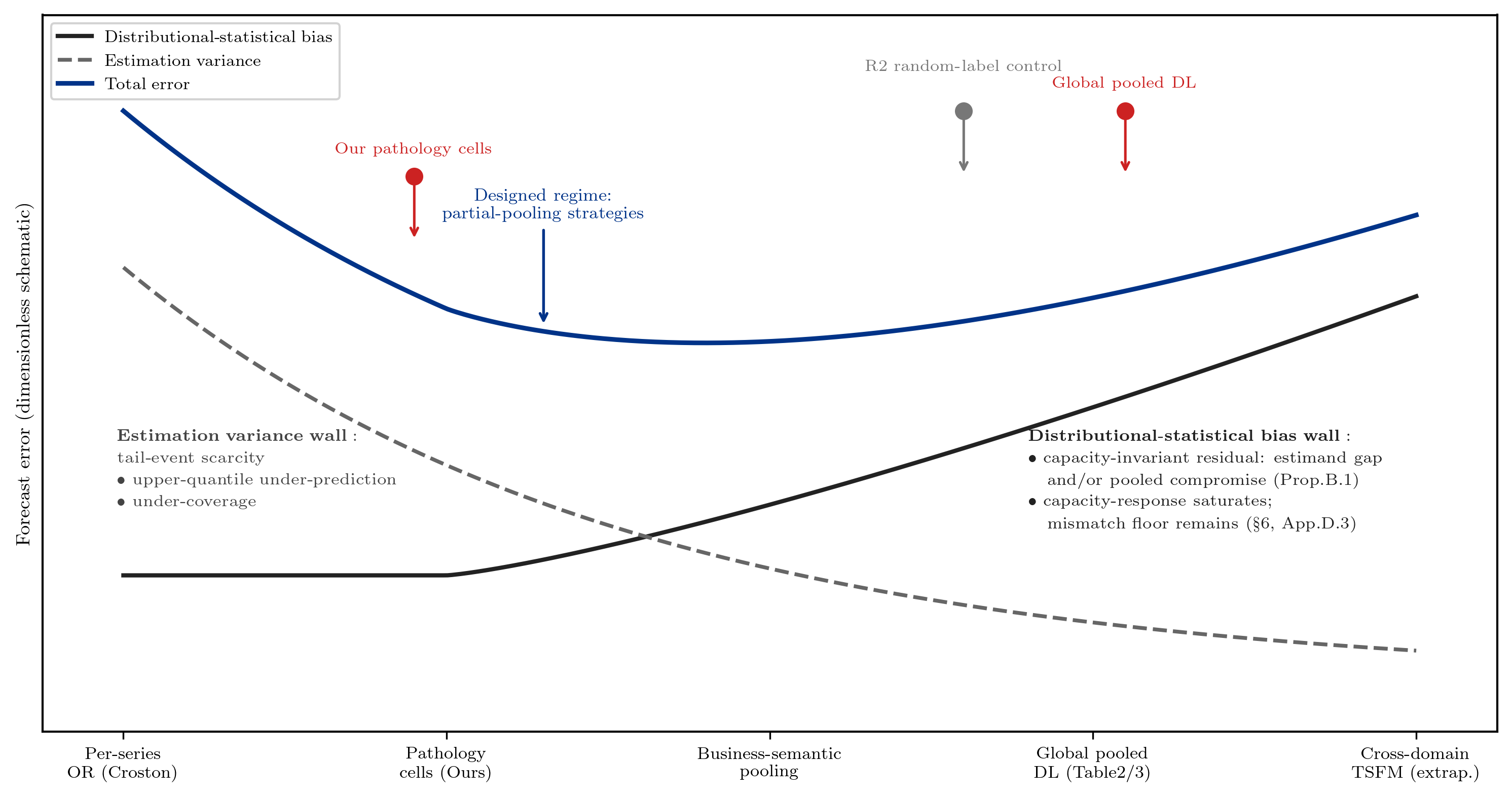}
\caption{Conceptual schematic of the aggregation-breadth trade-off between
distributional-statistical bias and estimation variance. The global-pooled,
random-cell (R2, random-label control), and pathology-cell markers correspond
to measured configurations (Tab.~\ref{tab:decompose_effect};
\S~\ref{app:quantile_spectra}, Tab.~\ref{tab:tab_loss_trajectory_grid}); remaining
markers are positional schematics without measured magnitude. Measured anchors
are consistent with, but do not locate, a total-error minimum near
pathology-aware subgrouping.}
\label{fig:agg_tradeoff}
\vspace{-15pt}
\end{figure}

Above we present a conceptual schematic of the qualitative bias-variance
trade-off under pooled multi-series forecasting; no closed-form curve is
derived---three operating points (global-pooled, random-cell R2,
pathology-cell) are measured (Tab.~\ref{tab:decompose_effect};
\S~\ref{app:quantile_spectra}). The trade-off characterizes settings with
strong cross-series distributional heterogeneity (zero-inflation, skewness,
high variability), where pooled training mixes distributions of divergent
statistical properties. For near-homogeneous collections (low PCS,
\S~\ref{app:pcs}), both walls recede and the trade-off flattens. It does not
apply to tasks dominated by temporal-structure challenges such as long-range
dependence, or to causal-inference problems unrelated to cross-series
distribution mixing. PCS serves as an empirical gauge of applicability:
predictions are load-bearing exactly where pathological coverage is high.

Under pooled empirical risk minimization, any loss optimizes for the
pooled-distribution optimum, which per Prop.~\ref{prop:decomposition}
constitutes a weighted compromise of subgroup-specific optima.
Accumulated distributional-statistical bias is largely invisible to the
training objective and poorly captured by aggregate metrics, as demonstrated
by quantile-coverage collapse that point-wise metrics reward
Fig.~\ref{fig:quantile_calibration}.
A further subtlety arises from the congruence between training-loss priors and
global evaluation metrics: even as aggregation pushes models toward the bias-wall of
Fig.~\ref{fig:agg_tradeoff}---the saturation point where more
global training no longer reduces regime bias---global scores stay
nearly unchanged, masking systematic deviations and motivating our RBV
diagnostic protocol.
Hierarchical partial-pooling approaches are designed to target this
intermediate regime; pathology-aware pooling to realize this operating point
is discussed in \S~\ref{sec:discussion}.
 
\section{Supplementary Formal Statements, Analytical Derivations and Robustness Validations}
\label{app:full}

This appendix provides supplementary formal derivations, analytical proofs, regularity validations, and robustness ablations to complement the main paper's theoretical framework and experimental conclusions.

\subsection{Loss--Marginal Conflict Verification \& Proposition Derivations}
\label{app:bias_verify}

\paragraph{Mechanistic verification.}
We provide case-by-case mechanistic verification for the core loss
misspecification claims (Tab.~\ref{tab:loss_misspecification}), confirming
that each foundational industrial pathology violates one or more of the
classical working conditions under which a single global loss prior is
statistically unambiguous, inducing a loss--prior/marginal conflict.
Under a symmetric marginal, the mean and median coincide and the choice
among mean-, median-, and quantile-consistent losses is immaterial; under
a continuous non-degenerate support, all implied quantiles are
well-defined and non-degenerate; and under homogeneous samples, uniform
weighting in empirical risk minimization weights all parts of the
marginal as intended. We verify that each foundational industrial
pathology breaks one or more of these conditions, so that the loss-implied
optimum becomes prior-dependent on the pathological marginal, inducing
distributional-statistical bias:

\begin{enumerate}[leftmargin=*]
    \item \textbf{Skewness Pathology ($|\gamma_1| \gg 0$):} MSE optimizes
    mean-centered residuals, while MAE optimizes median-based estimation.
    Severe skewness diverges mean and median values, leaving the two
    loss-implied optima distinct on the same marginal. Violation of the
    symmetric residual assumption creates a non-zero divergence between
    the two statistical optima, exposing the prior-dependence of the
    loss-implied optimum.

    \item \textbf{Zero-Inflation Pathology (High ZIR):} High zero-inflation
    violates the continuous support assumption of regression losses. Low
    quantile targets degenerate to zero, and MAE median estimation
    collapses toward trivial zero samples. Median-based optimization
    prioritizes fitting dominant zero observations, thereby discarding the
    high-demand burst mass, distorting the loss-implied optimum on the
    pathological marginal.

    \item \textbf{Within-series Scale Variability (High CV):} Canonical
    losses apply uniform sample weighting during risk minimization. High
    CV creates extreme imbalance between trivial low-magnitude samples
    and heavy-tailed demand bursts. Uniform weighting dilutes gradient
    contributions of high-impact samples, causing under-fitting of
    critical cases and prioritizing average accuracy over tail fidelity,
    so the fixed weighting prior mis-serves the pathological marginal.
\end{enumerate}

All three foundational pathologies break one or more of the classical
working conditions under which a single global loss prior is statistically
unambiguous, inducing a systematic loss--prior/marginal conflict and
completing the mechanistic verification.

\subsection{Standard Regularity Assumptions for Regime Trend Analysis}
\label{app:regularity}

To eliminate ambiguity in distribution perturbation and error decomposition, we state standard, industrially valid regularity assumptions for rigorous generalization error bounding, all naturally satisfied by retail time-series data and canonical loss functions.

\begin{assumption}[Bounded Industrial Demand Support]
\label{ass:bounded}
The industrial demand variable satisfies $y_t \in [0, Y_{\max}]$ for a finite constant $Y_{\max} < \infty$. This holds for all retail sales records, as real-world transaction volumes are strictly bounded.
\end{assumption}

\begin{assumption}[Lipschitz-Smooth Loss Function]
\label{ass:lipschitz}
Any fixed-form loss $\mathcal{L}(y,\hat{y})$ considered (MSE, MAE, Quantile, Huber, Tweedie, WMAPE) is $L$-Lipschitz continuous with respect to the prediction $\hat{y}$ over the bounded support $[0, Y_{\max}]$. Bounded Lipschitzness guarantees stable gradient behavior and controllable error perturbation under distribution shifts.
\end{assumption}

\begin{assumption}[Pathology-Induced Absolute Continuous Measure Transformation]
\label{ass:measure}
For each coupled pathological regime $p \in \mathcal{P}_{\mathrm{couple}}$, define the distribution shift operator $T_p$ as an absolutely continuous measure transformation that maps the benign baseline distribution $\mathcal{D}_{\mathrm{benign}}$ to the distorted pathological distribution $\mathcal{D}_p$. The transformation is parameterized by a normalized pathological intensity $\lambda(p) \in [0,1]$, where $\lambda(p)=0$ denotes benign baseline and $\lambda(p)=1$ denotes maximum pathological distortion. The Radon--Nikod\'{y}m derivative $d\mathcal{D}_p / d\mathcal{D}_{\mathrm{benign}}$ exists and is bounded for all valid $\lambda(p)$.
\end{assumption}

To make the contamination structure explicit, we first state the assumption
under which Prop.~\ref{prop:monotonic} is formalized, and then supplement
its detailed derivation to verify the monotonic regime-dependent failure
trend observed in experiments.

\begin{assumption}[Pathology as Distribution Contamination]\label{ass:contamination}
Each pathological family $\{\mathcal{D}_\lambda\}_{\lambda\in[0,1]}$ is
parameterized by contamination of the benign baseline:
$\mathcal{D}_\lambda=(1-\lambda)\,\mathcal{D}_{\mathrm{benign}}+\lambda\,\mathcal{D}_{\mathrm{path}}$,
where $\mathcal{D}_{\mathrm{path}}$ is a fixed pathological distribution on
the same support $[0,Y_{\max}]$. The intensity $\lambda$ is interpretable as
the fraction of pathological mass (e.g., zero-inflation rate, mixture
weight), and all expectations under $\mathcal{D}_\lambda$ are well-defined by
linearity of mixture distributions, with no ordering of Radon--Nikod\'{y}m
derivatives required.
\end{assumption}

\begin{proposition}[Loss‑relative error separation]\label{prop:decomposition}
Let $\varphi^*(\cdot\,|\,y_{<t})$ denote the deployment‑target functional
(e.g., the conditional mean under mean‑oriented operational cost, a
quantile under service‑level cost). For any loss $\mathcal{L}$ in the
zoo, define its statistically optimal forecast
$\hat{y}^{\mathrm{stat}}_{\mathcal{L}}(y_{<t}) = \arg\min_{\hat{y}}
\mathbb{E}_{y_t\sim P(\cdot|y_{<t})}\mathcal{L}(y_t,\hat{y})$
(the conditional functional elicited by $\mathcal{L}$: mean for MSE,
median for MAE, $q_\tau$ for pinball), and the loss‑relative bias
$b_{\mathcal{L}} := \hat{y}^{\mathrm{stat}}_{\mathcal{L}} - \varphi^*$.
The realized forecast decomposes relative to the deployment target as
$\hat{y}_t - \varphi^*(y_{<t}) = b_{\mathcal{L}}(y_{<t}) + e_t$, where
$e_t$ is the residual temporal‑fitting error. Under perfect temporal
fitting ($e_t = 0$), the remaining deviation $b_{\mathcal{L}}$ is
mechanistically independent of temporal modeling: a fixed loss prior
whose implied estimand misaligns with the deployment target induces
prediction bias that temporal modeling cannot remove. Two readings
follow. (i) If the target is the conditional mean, then
$b_{\mathrm{MSE}} \equiv 0$ and MSE exhibits no estimand‑mismatch
component; its regime‑level bias in Tab.~2 is then entirely the
aggregation component arising because pooled training
shifts the shared optimum toward a demand‑weighted cross‑regime
compromise that departs from every per‑regime mean. (ii) For losses
whose estimand differs from the target (e.g., MAE against a mean
target), $b_{\mathcal{L}}$ is nonzero already per distribution and is
further displaced under pooling. In both cases the
distributional‑statistical channel constitutes a loss--data mismatch
distinct from temporal‑structural bias.
\end{proposition}

\begin{proposition}[Risk of a fixed predictor under evaluation-distribution contamination]\label{prop:monotonic}
Fix any predictor $\hat{y}$, any loss $\mathcal{L}$ whose expectations below exist, and write
\[
R_{\mathrm{benign}}:=\mathbb{E}_{\mathcal{D}_{\mathrm{benign}}}[\mathcal{L}(y,\hat{y})],
\qquad
R_{\mathrm{path}}:=\mathbb{E}_{\mathcal{D}_{\mathrm{path}}}[\mathcal{L}(y,\hat{y})].
\]
Under Assumption~\ref{ass:contamination}, the risk under the contaminated mixture
\begin{equation}
R(\lambda)
:=\mathbb{E}_{\mathcal{D}_\lambda}[\mathcal{L}(y,\hat{y})]
=(1-\lambda)\,R_{\mathrm{benign}}+\lambda\,R_{\mathrm{path}},
\label{eq:riskaffine}
\end{equation}
is affine in $\lambda$ with slope $R_{\mathrm{path}}-R_{\mathrm{benign}}$, and is
non-decreasing (strictly increasing) in $\lambda$ whenever
$R_{\mathrm{path}}\geq\,(>)\;R_{\mathrm{benign}}$,
i.e., whenever the pathological distribution raises the expected loss of the
fixed predictor above its benign-data level.
\end{proposition}

\begin{proof}
By linearity of expectation under the mixture
$\mathcal{D}_\lambda=(1-\lambda)\mathcal{D}_{\mathrm{benign}}+\lambda\mathcal{D}_{\mathrm{path}}$
(Assumption~\ref{ass:contamination}),
\[
R(\lambda)=(1-\lambda)\,R_{\mathrm{benign}}+\lambda\,R_{\mathrm{path}}
=R_{\mathrm{benign}}+\lambda\big(R_{\mathrm{path}}-R_{\mathrm{benign}}\big),
\]
which is affine in $\lambda$ with slope $R_{\mathrm{path}}-R_{\mathrm{benign}}$;
the slope is non-negative (positive) exactly under the stated condition.
\end{proof}

\noindent\textbf{Scope.} This analysis holds for a fixed predictor $\hat{y}$:
it characterizes how contamination of the \emph{evaluation} distribution
affects a given predictor, and does not characterize outcomes when the
predictor is re-optimized on training data with varying pathology fractions.

\paragraph{Verifying the premise.}
The premise is checkable for standard losses. For MSE, writing
$\mu_\lambda=\mathbb{E}_{\mathcal{D}_\lambda}[y]$ and
$\sigma^2_\lambda=\mathrm{Var}_{\mathcal{D}_\lambda}(y)$, the bias--variance
identity
$\mathbb{E}_{\mathcal{D}_\lambda}[(y-\hat{y})^2]=\sigma^2_\lambda+(\mu_\lambda-\hat{y})^2$
shows that the premise holds whenever the pathological family raises total
dispersion or shifts the mean away from the benign-data level toward which
the predictor was tuned. For the pinball loss at level $\tau$, the identity
$\mathbb{E}[\rho_\tau(y-\hat{y})]=\tau(\mu-\hat{y})+\mathbb{E}[(\hat{y}-y)_+]$
reduces the premise to a two-term comparison between mean shift and
below-prediction shortfall; zero inflation, e.g., adds mass at $y=0$ and
raises $\mathbb{E}[(\hat{y}-y)_+]$ for any $\hat{y}>0$. We do not claim the
premise holds universally; rather, we verify it in our experiments:
Tab.~\ref{tab:regime_perf} reports per-regime risks of predictors trained
on pooled data, and Fig.~\ref{fig:dose_response_ablation} shows the dose--response
trend, jointly confirming the premise for all coupled pathologies and losses
studied.

\subsection{RBV: Template, Interpretation, and Uncertainty Quantification}
\label{app:rbv_stats}
\paragraph{Unified template and three instantiations.}
All RBV variants used in this paper are instances of one template. Let
$\{I_r\}_{r=1}^{R}$ be a fixed assignment of test series to regimes (a
series may belong to several regimes; membership indicators
$m_r(i)\in\{0,1\}$ replace the partition weights below where overlap
occurs), let $\phi_r = \phi\big((\hat y_i, y_i)_{i \in I_r}\big)$ denote
any regime-level functional of predictions and realized labels, and let
$w_r \propto \sum_i m_r(i)$ be count-based weights, normalized so that
$\sum_r w_r = 1$; for disjoint partitions this reduces to $w_r = |I_r|/N$.
For pathology bias-RBV the centering is the pooled bias of
\S\ref{sec:rbv_def}, i.e., $w_r \propto \alpha_r$ (the absolute-demand
mass of cell $r$) rather than the count share.
The centered deviation vector and its aggregation are defined as
\begin{equation}
\Delta B_{\phi,r} \;=\; \phi_r - \bar\phi,
\qquad
\bar\phi \;=\; \textstyle\sum_{r=1}^{R} w_r\,\phi_r,
\qquad
\mathrm{RBV}_{\phi} \;=\; \big\|\Delta B_{\phi}\big\|_2 .
\label{eq:rbv-template}
\end{equation}
The metric-agnostic protocol of \S\ref{sec:metric_agnostic_protocol} is precisely the
freedom to instantiate the functional slot $\phi$; this appendix makes the
three instantiations used in the paper explicit
(Tab.~\ref{tab:rbv-family}).
\begin{table}[H]
\caption{The RBV family: one template, three instantiations. All variants
share the same centering and $\ell_2$ aggregation, and differ only in the
base functional $\phi_r$. $E_r(\tau)$ is the empirical exceedance under
the mid-convention
$\widehat{\Pr}[\hat q_\tau > y] + \tfrac{1}{2}\,\widehat{\Pr}[\hat q_\tau = y]$
(Figs.~\ref{fig:quantile_calibration}, \ref{fig:fig_m5_quantile_calibration});
$\hat S_r = \sum_{i \in I_r} \hat y_i$ is the regime-level predicted total, and
$B_r = \big(\sum_{i \in I_r}(\hat y_i - y_i)\big)/\big(\sum_{i \in I_r}|y_i|\big)$
denotes the regime-wise pooled bias of \S\ref{sec:rbv_def}.}
\label{tab:rbv-family}
\centering
\small
\setlength{\tabcolsep}{4pt}
\begin{tabularx}{\textwidth}{@{}p{2.8cm}>{\centering\arraybackslash}p{3.6cm}p{1.8cm}X@{}}
\toprule
\textbf{Variant} & \textbf{Base functional $\phi_r$} & \textbf{Units} & \textbf{Question answered} \\
\midrule
Pathology bias-RBV (default, \S\ref{sec:rbv_def})
  & $\frac{\sum_{i \in I_r} (\hat y_i - y_i)}{\sum_{i \in I_r} |y_i|}$
  & demand units
  & Which regime is systematically over-/under-estimated, and by how much? \\
\addlinespace[3pt]
Quantile RBV
  & $\phi_r(\tau) = E_r(\tau) - \tau$
  & probability
  & Where -- and at which $\tau$ -- is calibration error concentrated? \\
\addlinespace[3pt]
Share-normalized RBV
  & $B_r / \hat S_r$
  & dimensionless (relative bias)
  & How is mismatch severity allocated per unit of business volume? \\
\bottomrule
\end{tabularx}
\end{table}

\paragraph{Reading guidance.}
Three caveats govern the use of the RBV family.
\emph{(i) No cross-variant comparison of magnitudes.} The three variants
measure different objects on different scales (Tab.~\ref{tab:rbv-family});
quantitative statements are always within-variant, and cross-variant
statements are ranking-level only.
\emph{(ii) Quantile RBV is informative only above the estimand floor.}
This per-$\tau$ reading is not a refinement but a necessity: quantiles
are not additive under aggregation---the sum of per-series
$\tau$-quantiles equals the $\tau$-quantile of the total only under
comonotonic dependence---so no single scalar RBV can summarize the
cross-regime bias structure of a quantile forecaster, and interpretation
is per-$\tau$ by construction.
For a regime with zero-inflation rate $\mathrm{zir}_r$ (the fraction of zero
realizations in regime $r$), a perfectly calibrated $\tau$-quantile
predictor degenerates to zero whenever
$\tau \le \mathrm{zir}_r$, and its expected exceedance under the
mid-convention equals $\mathrm{zir}_r/2$ regardless of fit quality -- the
estimand floor of \S\ref{sec:rbv_def}, annotated in
Figs.~\ref{fig:quantile_calibration}--\ref{fig:fig_m5_quantile_calibration}.
Hence $\Delta B_{E(\tau),r}$
with $\tau \le \mathrm{zir}_r$ reflects estimand structure rather than
miscalibration; quantile RBV is interpreted only on the band
$\tau > \max_{r \in \text{pathological regimes}} \mathrm{zir}_r$, with floor-pinned
components reported as reference lines and excluded from interpretation.
\emph{(iii) Conservation holds under the centering weights only.} By
construction the \emph{weighted} sum $\sum_r w_r\,\Delta B_{\phi,r} = 0$
(the unweighted sum $\sum_r \Delta B_{\phi,r}$ is generally non-zero), and
the bidirectional sign pattern is one estimand gap read from its two sides
\emph{under the weights $w$ that define $\bar\phi$}. For the share-normalized variant the
centering weights $w_r$ must be distinguished from the normalization
denominators $\hat S_r$: if $\bar\phi$ is formed under a weighting other
than $w$ (e.g., volume shares), the conservation identity and the
bidirectional reading must be restated under \emph{those} weights; mixing
the two weight systems silently breaks the identity. A guard
$\hat S_r \ge \varepsilon$ excludes degenerate all-zero cells from the
share-normalized variant.

\paragraph{Remediating a high RBV reading.}
RBV is a diagnostic, not a verdict; a high value prescribes a
regime-localized intervention path. We recommend a three-step reading.
(1)~\emph{Attribute}: rank regimes by $|B_r - B_{\mathrm{global}}|$ (equivalent to $|\Delta B_{\phi,r}|$ under pathology-bias RBV); the
dominant cell identifies which pathology the pooled model sacrifices.
(2)~\emph{Classify}: distinguish the two layers using a diagnostic
per-regime retrain test; if within-cell bias persists
(Fig.~\ref{fig:capacity_mae}), the residual is estimand mismatch
(Layer~1), otherwise the pooled compromise itself (Layer~2,
cf.~Fig.~\ref{fig:capacity_mse}).
(3)~\emph{Intervene}: for Layer~1, adopt a curvature-appropriate estimand
(quantile at the operational critical ratio, Tweedie for intermittent
positive demand); for Layer~2, dissolve the pooling via regime decomposition
or partial pooling with strength adapted to per-cell volume
(Fig.~\ref{fig:agg_tradeoff}). Capacity increases reduce RBV only down to a
loss-dependent floor (the pooled compromise) and cannot remove it; they should not be the first response.
\paragraph{Uncertainty quantification.} To verify that the regime-level
diagnostics on M5 are not artifacts of a particular train/test split, we
recompute $B_r$ and RBV under SKU-level bootstrap resampling
($B{=}1000$ replicates; SKUs are the i.i.d.\ unit, statistics are
recomputed by pooling samples within each replicate, exactly matching
the evaluation convention of Sec.~\ref{app:cross_domain}). Point estimates
reproduce the main table to three decimals
(Tab.~\ref{tab:rbv_ci}). All cell-level biases are significant at
the 1\% level; RBV 95\% confidence intervals are tight (half-width
${\le}0.05$, e.g., $0.409\,[0.391,0.427]$ for MAE and
$0.101\,[0.080,0.123]$ for the quantile compromise at
$\tau{=}0.6$), far smaller than between-cell gaps, so the loss
ranking by RBV is stable. The single interval containing zero is the
pooled bias of the $\tau{=}0.6$ quantile compromise
($-0.001\,[-0.010,0.009]$), consistent with its near-unbiased
aggregated profile.

\begin{table}[t]
\centering
\caption{Bootstrap uncertainty (SKU-level, $B{=}1000$) for the pooled
bias $B$ and regime bias variation (RBV), single-origin M5 protocol.
All RBV intervals exclude zero.}
\label{tab:rbv_ci}
\small
\begin{tabular}{lcc}
\toprule
Loss & $B$ [95\% CI] & RBV [95\% CI] \\
\midrule
MAE      & $-0.257$ $[-0.267,-0.248]$ & $0.409$ $[0.391,0.427]$ \\
MSE      & $ 0.058$ $[ 0.048, 0.068]$ & $0.419$ $[0.387,0.450]$ \\
Huber    & $-0.107$ $[-0.115,-0.098]$ & $0.241$ $[0.215,0.265]$ \\
Tweedie  & $ 0.054$ $[ 0.044, 0.063]$ & $0.378$ $[0.350,0.407]$ \\
Poisson  & $ 0.058$ $[ 0.049, 0.068]$ & $0.477$ $[0.445,0.507]$ \\
Quantile $\tau{=}0.1$ & $-0.859$ $[-0.866,-0.852]$ & $0.229$ $[0.218,0.242]$ \\
Quantile $\tau{=}0.2$ & $-0.740$ $[-0.749,-0.730]$ & $0.373$ $[0.359,0.390]$ \\
Quantile $\tau{=}0.3$ & $-0.609$ $[-0.619,-0.599]$ & $0.464$ $[0.447,0.479]$ \\
Quantile $\tau{=}0.4$ & $-0.451$ $[-0.461,-0.442]$ & $0.494$ $[0.477,0.510]$ \\
Quantile $\tau{=}0.6$ & $-0.001$ $[-0.010, 0.009]$ & $0.101$ $[0.080,0.123]$ \\
Quantile $\tau{=}0.7$ & $ 0.336$ $[ 0.324, 0.349]$ & $0.588$ $[0.544,0.628]$ \\
Quantile $\tau{=}0.8$ & $ 0.739$ $[ 0.721, 0.757]$ & $1.305$ $[1.243,1.369]$ \\
Quantile $\tau{=}0.9$ & $ 1.348$ $[ 1.322, 1.375]$ & $2.088$ $[1.996,2.172]$ \\
\bottomrule
\end{tabular}
\end{table}




\paragraph{Computation.}
RBV is computed in two passes over the evaluation pool; pseudocode is
given in Alg.~\ref{alg:rbv}. Pass~1 accumulates, per regime cell
$r \in \mathcal{R}$, the residual mass $\rho_r = \sum_{i \in r}
(\hat y_i - y_i)$ and the absolute demand mass $\alpha_r = \sum_{i \in
r} |y_i|$, where $i$ ranges over all finite (series, horizon) pairs in
the pooled evaluation window; the pooled bias is $B = (\sum_r \rho_r) /
(\sum_r \alpha_r)$. Pass~2 forms the centered vector
$B_r - B$ with $B_r = \rho_r / \alpha_r$ and returns its unweighted
$\ell_2$ norm $\mathrm{RBV} = \bigl(\sum_r (B_r - B)^2\bigr)^{1/2}$.
Cells with $\alpha_r = 0$ (or fewer samples than the censoring floor)
are excluded up front and RBV is reported as undefined if any of the
four retained cells is empty, so the statistic is never computed on a
partial regime partition. Both passes are $O(N)$ in the number of
evaluated samples $N$ (a single streaming sweep with four accumulators
per cell; no storage of per-sample residuals is required), and the
entire M5 evaluation ($N = 187{,}222$) completes in well under a second
on a single CPU core. The same two-pass accumulation yields per-cell
bootstrap replicates directly: resampling series rather than samples
amounts to re-weighting the per-series contributions to
$(\rho_r, \alpha_r)$, so Tab.~\ref{tab:rbv_ci} costs $O(B \cdot
N_{\mathrm{sku}})$, not $O(B \cdot N)$.

\begin{algorithm}[t]
\caption{Two-pass computation of pooled bias and RBV}
\label{alg:rbv}
\begin{algorithmic}[1]
\Require predictions $\{\hat y_i\}$, truths $\{y_i\}$, regime map
$c(i) \in \{\mathrm{benign}, \mathrm{zir}, \mathrm{skew}, \mathrm{cv}\}$
\State $(\rho_r, \alpha_r) \gets (0, 0)$ for all $r$ \Comment{Pass 1:
accumulate per-cell masses}
\ForAll{finite evaluation pairs $i$ with $y_i$ observed}
    \State $\rho_{c(i)} \mathrel{+}= \hat y_i - y_i$;\quad
           $\alpha_{c(i)} \mathrel{+}= |y_i|$
\EndFor
\State $B \gets \bigl(\sum_r \rho_r\bigr) \big/ \bigl(\sum_r \alpha_r\bigr)$
\ForAll{$r$} \Comment{Pass 2: centered $\ell_2$ norm}
    \State $B_r \gets \rho_r / \alpha_r$
\EndFor
\State \Return $B$,\; $\mathrm{RBV} \gets \bigl(\sum_r (B_r - B)^2\bigr)^{1/2}$
\end{algorithmic}
\end{algorithm}

\subsection{Threshold Perturbation Sensitivity Analysis}
\label{app:threshold}

We recompute the M5 regime map and RBV under one-at-a-time threshold
scans ($t_{\rm zir}\in[0.10,0.50]$, $t_{\rm skew}\in[0.25,2.00]$,
$t_{\rm cv}\in[1.00,3.00]$), $7{\times}7$ pairwise joint grids, and
boundary stress tests (all thresholds doubled/halved). Three findings.
(i)~\emph{Rankings and signatures are threshold-stable:} the loss-family
ordering, the sign pattern of per-cell bias, and the U-shaped
RBV($\tau$) structure are invariant across all 39 settings; the skew
threshold is the least influential (RBV spread ${<}0.05$ for 12 of 13
losses).
(ii)~\emph{Magnitude tracks cell extremity:} RBV's sensitivity is
dominated by $t_{\rm cv}$ ($0.25$--$2.48$ spread), because a higher
threshold shrinks the CV cell toward the distribution's tail
($10{,}746 \to 89$ SKUs at the loose boundary), mechanically inflating
$|B_{\rm cv}|$; Cross-setting RBV comparisons should therefore always be read against
the cell sizes that generate them (per setting: Fig.~\ref{fig:app_rbv_threshold_sens}
and above); in the main tables the regime map is fixed by protocol, so
cell sizes are constant across columns and omitted.
(iii)~\emph{The compromise optimum is robust to plausible perturbation:}
ql60 attains the minimum RBV under $\pm20\%$ threshold shifts in all
settings; only under the 2$\times$ extreme tightening does Huber join it
in a low-RBV plateau ($0.046$ vs.\ $0.059$).

\begin{figure}[t!]
  \centering
  \includegraphics[width=\textwidth]{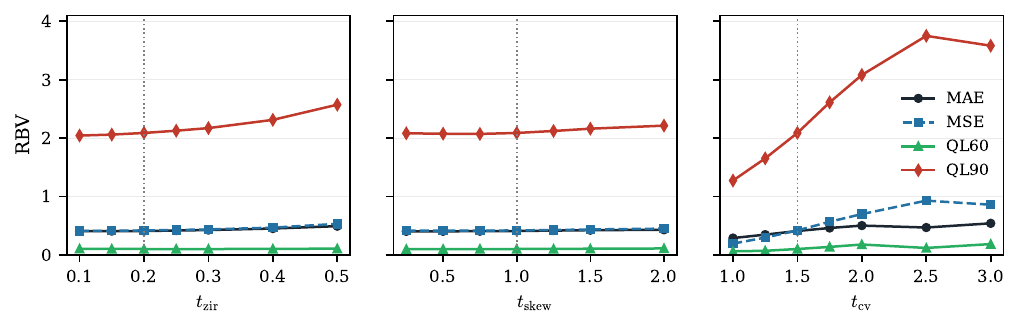}
  \caption{One-at-a-time threshold sensitivity of regime-decomposed RBV on M5.
  Each panel varies one threshold ($t_{\rm zir}$, $t_{\rm skew}$, $t_{\rm cv}$)
  with the other two fixed at baseline ($0.20/1.00/1.50$, dotted line);
  the remaining 9 losses behave similarly and are omitted for clarity.
  Rankings and the U-shaped RBV--$\tau$ structure are invariant across all
  threshold settings; the dominant sensitivity is along $t_{\rm cv}$, whose
  spread tracks the extremity of the CV cell (89--10{,}746 SKUs across
  settings) rather than the threshold definition itself.
  Note: for this sensitivity analysis, regime cells are formed directly
  from the continuous pathology statistics (ZIR rate, skewness, CV) without
  applying the \S\ref{sec:taxonomy} thresholds, so RBV magnitudes are not
  directly comparable to Table~\ref{tab:regime_perf} (which uses the
  thresholded regime labels); only rankings and shapes are comparable.}
  \label{fig:app_rbv_threshold_sens}
\end{figure}

\subsection{Metric-Agnostic Validation}
\label{app:metric_agnostic}
To verify that the regime-dependent failure patterns are intrinsic to loss-data distributional mismatch rather than artifacts of the evaluation metric, we reproduce the regime-decomposed evaluation under three distinct functionals: MSE, MAE, and pinball loss.

Across all metrics, the relative failure rankings remain invariant: MSE-trained models exhibit the most severe degradation under high zero-inflation rate (ZIR); MAE-trained models degrade most severely under strong skewness; and Quantile, Huber, and Tweedie losses show prominent fragility under high conditional variability (CV). Absolute values scale with each metric’s sensitivity to tail risk (e.g., MSE amplifies burst errors quadratically, while MAE treats them linearly), but the structural misspecification signatures are metric-independent. This confirms that the regime-specific failures reported in Sec~\ref{sec:rbv_def} originate from the inherent mismatch between canonical loss statistical priors and pathological data distributions, not from the choice of evaluation functional.

\subsection{Pathological Coverage Score (PCS): Dataset Auditing Metric}
\label{app:pcs}

This appendix formalizes the Pathological Coverage Score used for the benchmark coverage audit
(App.~\ref{app:cross_domain}, Tab.~\ref{tab:benchmark_pcs}). To quantitatively audit dataset
industrial validity, we define PCS ($0$--$1$), measuring comprehensive pathological
regime coverage:
\begin{equation}
\text{PCS} = 1 - \frac{1}{3}\sum_{d \in \{\text{ZIR},\text{skew},\text{CV}\}} \frac{|\{i:\text{regime}_d(i)=\text{low}\}|}{N}
\label{eq:pcs}
\end{equation}
Higher PCS indicates a richer pathological spectrum; we measure from the
low end to maximize discriminative range in overwhelmingly benign
industrial data, complementing regime-decomposed evaluation which
isolates severe high-pathology misalignments.

\paragraph{Weight robustness.} PCS adopts equal weighting for the three
pathological dimensions by default. We ablate alternative weight
configurations $(0.2, 0.4, 0.4)$ and $(0.4, 0.3, 0.3)$ and find that
weight perturbations do not change the relative PCS ordering of
mainstream datasets, confirming that the audit conclusions are robust
to the weighting scheme.

\subsection{Synthetic Illustration for Prop.~\ref{prop:decomposition}}
\label{app:proof_separation}
\begin{proof}[Proof of Prop.~\ref{prop:decomposition}]
\[
\mathcal{B}_{\text{temp}}=\hat{y}_{t}-\mathbb{E}\big[y_{t}\mid y_{<t}\big] = 0.
\]
However, canonical losses embed fixed statistical priors and optimize static distributional functionals (mean for MSE, median for MAE) that are prior-dependent rather than data-determined: under non-Gaussian industrial pathologies, distinct losses induce distinct optima on the same marginal.
This demonstrates that distributional-statistical mismatch induces systematic estimand mismatch independently of temporal modeling quality.
\end{proof}

In these synthetic illustrative experiments, we do not train any forecasting model. Instead, we directly adopt the statistically-optimal closed-form predictor for each loss under the given pathological marginal distribution. This construction enforces \(B_{\mathrm{temp}}=0\) by design: temporal-structural bias is identically zero, isolating purely distributional-statistical misspecification effects. Real-world neural-model experiments (Tab.~\ref{tab:regime_perf}) later demonstrate how these effects manifest alongside temporal fitting errors in practical settings.

We present two symmetric discrete industrial distributions to illustrate the bidirectional pathology mechanism.
We emphasize that these deliberately simplified two-point cases are pedagogical illustrations of the core mechanism, 

\textbf{Case A (zero-dominant demand):}
$P(y=0)=0.9,\; P(y=100)=0.1$. MSE and MAE yield statistically optimal
estimates of $10$ and $0$, respectively: the two loss-implied estimands
differ by an order of magnitude on the same distribution, and their
ranking under quantile evaluation flips at the staircase boundary below.
This distribution produces degenerated quantile solutions:
\[
Q_\tau =
\begin{cases}
0, & \tau \in [0, 0.9], \\
100, & \tau \in (0.9, 1].
\end{cases}
\]
Small perturbations near the mass boundary trigger abrupt quantile
prediction jumps.

\textbf{Case B (burst-dominant demand):}
$P(y=0)=0.1, P(y=100)=0.9$. Statistically optimal MSE ($90$) and MAE ($100$) predictions again diverge by construction, confirming that the estimand is prior-dependent.

\begin{figure}[h!]
\centering
\includegraphics[width=\linewidth]{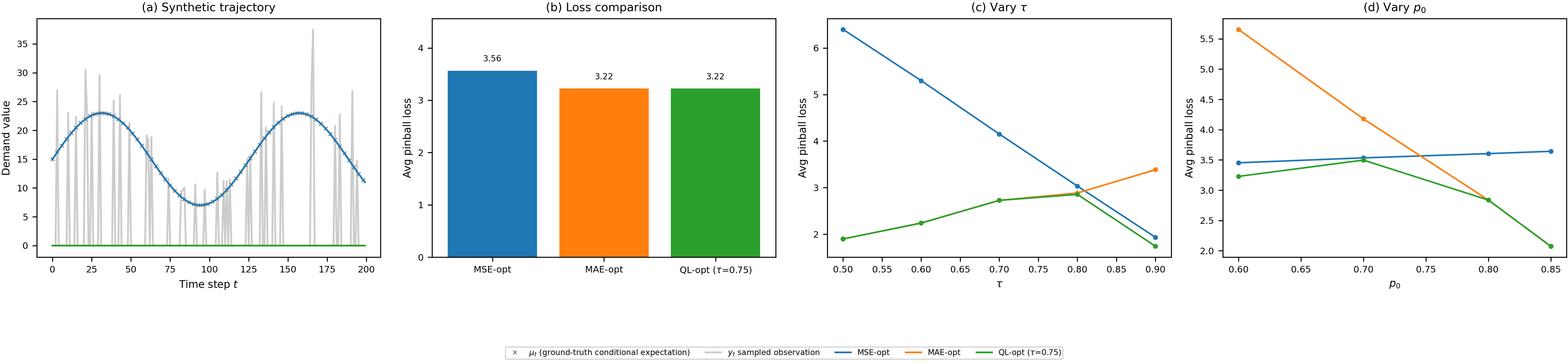}
\caption{Synthetic illustration and parameter sensitivity for
Prop.~\ref{prop:decomposition}.
(a) Example trajectory under perfect temporal modeling $\mathcal{B}_{temp}=0$.
(b) Quantitative comparison of loss-implied predictors (MSE-opt, MAE-opt,
and a fixed-quantile target) under pinball evaluation.
(c) Ablation over quantile level $\tau$: the ranking of predictors flips
as the target functional varies, illustrating that the optimal estimator
is prior-dependent.
(d) Ablation over zero-inflation rate $p_0$.
This figure serves as an illustrative synthetic example; formal statements
and proofs are given in the main text.}
\label{fig:theorem_synth}
\end{figure}

\begin{figure}[h!]
\centering
\includegraphics[width=\linewidth]{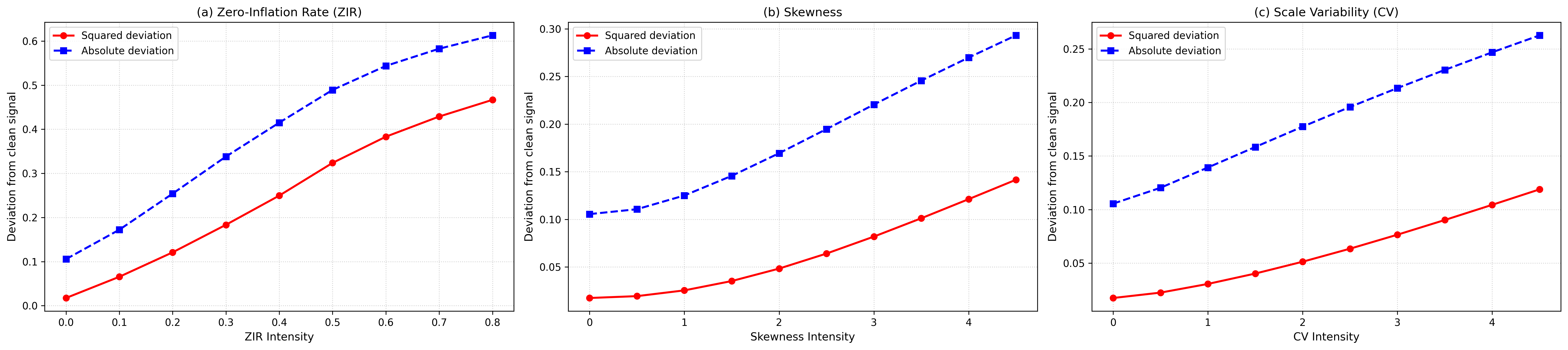}
\caption{Dose-response ablation over three core pathological dimensions:
zero-inflation rate (ZIR), skewness, and within-series scale variability (CV).
We report squared and absolute deviation of the fitted predictor from the
clean conditional trajectory as pathological intensity increases.
Consistently, both deviation metrics rise monotonically with pathological
severity, demonstrating that pathological marginals systematically corrupt
the recovery of the underlying signal.
This synthetic experiment complements the theoretical analysis in
Prop.~\ref{prop:decomposition}.}
\label{fig:dose_response_ablation}
\end{figure}
\FloatBarrier
These two symmetric cases demonstrate that the divergence of loss-implied
estimands originates from fixed loss statistical priors, independent of both
optimization quality and the direction of data sparsity.
Prop.~\ref{prop:decomposition} establishes that distributional-statistical
misspecification is a standalone failure mode even under perfect temporal
modeling.
Fig.~\ref{fig:theorem_synth} provides synthetic trajectory visualization,
a quantitative comparison across loss-implied predictors, and
hyper-parameter ablation over quantile level $\tau$ and zero-inflation
rate $p_0$.
Complementing this sensitivity study, Fig.~\ref{fig:dose_response_ablation}
shows the dose-response behaviour for the three foundational pathological
dimensions: zero-inflation rate (ZIR), skewness, and within-series scale
variability (CV).
As pathological intensity gradually increases, both squared and absolute
deviation from the clean signal grow monotonically.
This reinforces our core claim: fixed-prior loss functions induce
systematic, prior-dependent prediction deviation under industrial-style
distributional pathologies, even when temporal patterns are perfectly
fitted ($\mathcal{B}_{temp}=0$).
This synthetic example illustrates the underlying mechanism: accurate
temporal fitting alone cannot eliminate deviation originating from
loss-data statistical mismatch.

The regime-decomposed protocol \S~\ref{sec:metric_agnostic_protocol}
accepts arbitrary base-level evaluation functionals as input.
In Step 3 (Metric-Agnostic Quantification), we first compute per-regime
bias terms $B_r$ using either standard symmetric metrics (MSE, MAE, WMAPE,
raw prediction bias) or asymmetric statistical criteria (quantile loss,
pinball); these per-regime outputs are then aggregated into the Regime-wise
Relative Bias Vector (RBV).
While resulting RBV magnitudes depend on the choice of base evaluation
functional, the \textit{qualitative pattern of regime-specific fragility
remains consistent}: the identity of which pathologies degrade
disproportionately under a given training loss is preserved.
Specific instantiations are provided in the open-sourced implementation.

\subsection{Boundary Conditions for Saturation vs.~Super-Additive Degradation}
\label{app:coupled_ablation}

We conduct controlled synthetic experiments to isolate the interaction
effect between co-occurring ZIR and CV pathologies, varying a shared
intensity parameter $\alpha$. The linear-sum baseline is defined as
$\mathcal{L}_{\mathrm{linear}}(\alpha) = \mathcal{L}_{\mathrm{ZIR}}(\alpha)
+ \mathcal{L}_{\mathrm{CV}}(\alpha) - \mathcal{L}_{\mathrm{benign}}$,
subtracting the benign reference to avoid double-counting.

Under low $\alpha$, coupled loss tracks the linear baseline closely.
Within intermediate ranges, it strictly exceeds the baseline across
total MSE, burst-sample MSE, and pinball loss at $\tau{=}0.75$
(Fig.~\ref{fig:coupled_zir_cv_synth}), demonstrating super-additive
degradation consistent with the \emph{disjoint-support coupling}
hypothesized in \S~\ref{sec:loss_misspecification_mapping}. high ZIR
collapses effective support toward zero while high CV demands tail
extrapolation, creating conflicting gradient regimes. At high $\alpha$,
both isolated pathologies drive predictions toward a near-random floor;
additional distortion yields diminishing marginal degradation, causing
saturation and eventual sub-additive tendencies.

This validates the boundary-condition analysis proposed in
\S~\ref{sec:loss_misspecification_mapping}, though we do not claim
generalization to all pathology pairs or real-world settings where
counterfactual isolation is infeasible.

\begin{figure}[H]
    \centering
    \includegraphics[width=0.98\textwidth]{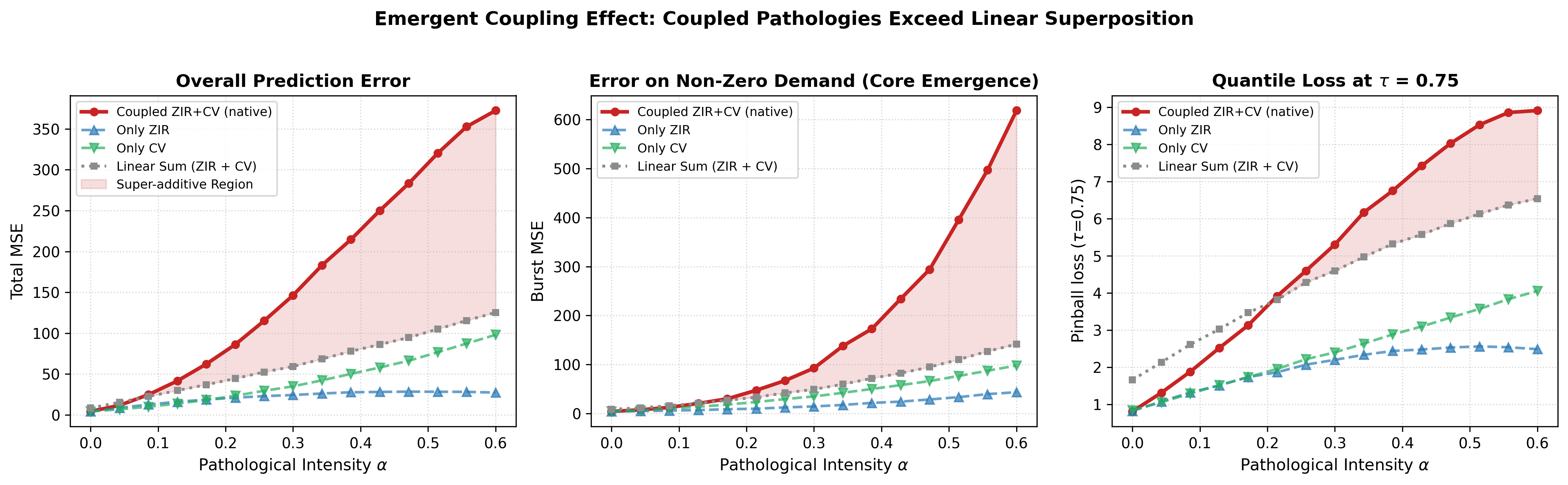}
    \caption{Coupled-pathology performance across varying $\alpha$. The linear-sum baseline aggregates excess losses from isolated-ZIR and isolated-CV runs. Pink shaded regions highlight regimes where coupled loss strictly exceeds this reference, indicating emergent super-additive degradation.}
    \label{fig:coupled_zir_cv_synth}
\end{figure}

\paragraph{Scope: diagnosis vs.\ mitigation.}
Robustification methods---distributionally robust optimization
\citep{blanchet2019quantifying,blanchet2019robust} and quantile regression
forests \citep{meinshausen06a}---re-engineer the objective to
withstand distributional perturbation. RBV is complementary by design:
a post-hoc diagnostic that localizes \emph{where} and \emph{why} a
forecaster misaligns, and can equally audit robust methods themselves;
pairing RBV-measured miscalibration with robust retraining is future work.

\section{CROSS-DOMAIN GENERALIZATION OF THE PATHOLOGY TAXONOMY}
\label{app:cross_domain}

\paragraph{Temporal structure of RetailShiftBench.}
The audit of Table~\ref{tab:benchmark_pcs} addresses pathological
coverage; we clarify the complementary question of temporal structure.
RetailShiftBench series are short ($89$ steps) and exhibit a stable
weekly periodicity without trend breaks or long-range drift. This is a
deliberate design choice: removing temporal non-stationarity as a
confound ensures that residual errors are attributable to
distributional misspecification rather than to distribution shift over
time. The aggregate daily sales over the full panel
(Fig.~\ref{fig:fig_rsb_trend}) illustrate this structure. External
validity across scale, window length, and pathology mix is established
on M5 (App.~\ref{app:m5_external}), which is adopted as the external validation
anchor above.

\begin{figure}[H]
  \centering
  \includegraphics[width=0.8\linewidth]{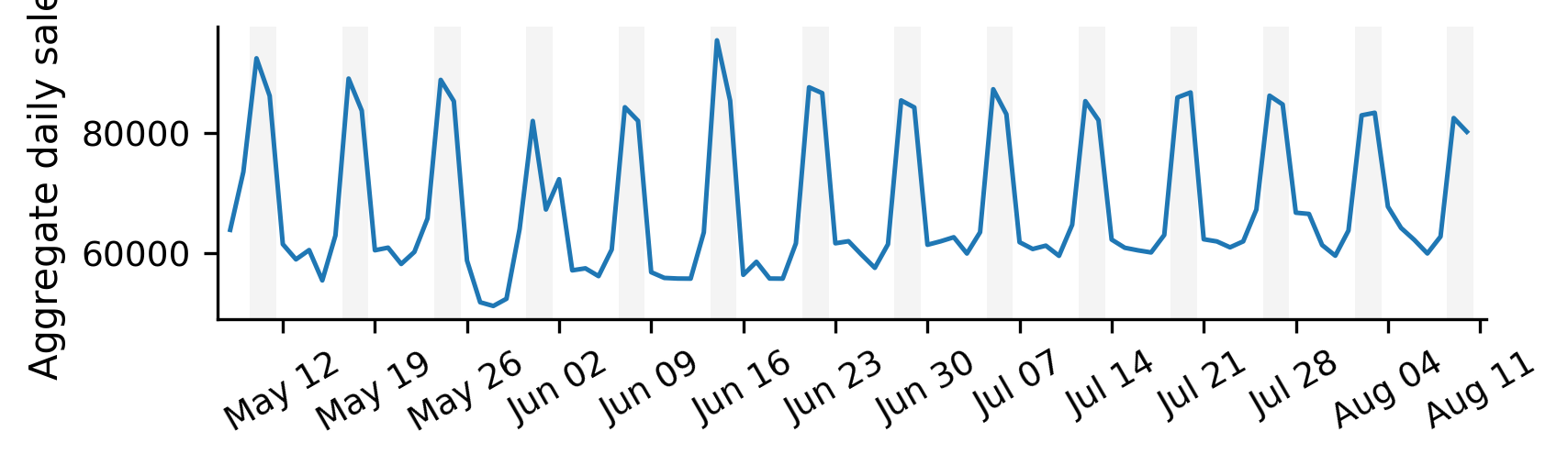}
  \caption{\textbf{RetailShiftBench temporal structure.} Aggregate daily sales across all 31{,}213 series over the full evaluation panel. Sales are dominated by a stable weekly periodicity with no trend breaks or regime shifts, indicating that the benchmark's difficulty stems from cross-sectional distribution heterogeneity rather than temporal non-stationarity.}
  \label{fig:fig_rsb_trend}
\end{figure}

Our hierarchical pathology taxonomy is designed as a \textbf{domain-agnostic}
statistical framework: the three foundational dimensions---distributional
skewness, zero-inflation, and within-series scale variability (CV)---violate
fixed statistical priors of canonical losses regardless of the business
guise under which they arise. Table~\ref{tab:cross_domain} instantiates the
taxonomy across four industrial domains, mapping each pathology and its
coupled forms onto domain-specific manifestations. The statistical
definition of each regime is unchanged across domains; only the threshold
calibration (\S~\ref{sec:taxonomy}) is domain-adaptive.

The same per-series binarization that defines RBV regime cells
(``pathological on dimension $d$ iff $\mathrm{regime}_d \neq \mathrm{low}$'')
supports a dataset-level audit: the Pathological Coverage Score
$\mathrm{PCS} = 1 - \frac{1}{3}\sum_{d} |\{i:\mathrm{regime}_d(i)=\mathrm{low}\}|/N$
reads directly as non-benign coverage, with higher values indicating a
richer pathological spectrum (\S~\ref{app:pcs}).
Table~\ref{tab:benchmark_pcs} audits mainstream forecasting benchmarks
under one standard. RetailShiftBench attains full-spectrum coverage by
construction; M4's frequency-stratified subgroups are overwhelmingly benign
and are therefore insufficient as industrial stress tests; M5 provides
adequate coverage and is adopted as the external validation anchor for the
regime-decomposed protocol (App.~\ref{app:m5_external}).

\textbf{M5 zero-run censoring.}
M5's long panel (2011--2016) contains assortment churn: items are listed,
delisted, and relisted, producing zero runs that reflect lifecycle structure
rather than demand pathology. Absent assortment flags, listing gaps are
unidentifiable from sales alone; we adopt a zero-run censoring convention:
all leading zeros before a SKU's first sale are truncated (the item is not
yet listed), and any continuous zero run longer than $k$ days ($k{=}14$;
sensitivity $k\in\{7,28\}$) is excluded from all pathology statistics
(ZIR, skewness, CV) and from model training. Under $k{=}14$,
27{,}725 of 30{,}490 SKUs (90.9\%) contain at least one censored segment,
totalling 10.9M censored time steps (18.46\% of the original panel);
PCS in Table~\ref{tab:benchmark_pcs} is computed on the censored panel.
Censoring lowers M5's PCS from 0.8905 (raw) to 0.7587, removing the
lifecycle-driven inflation; the censored-panel PCS sits in the same band
as RetailShiftBench (0.7194), supporting cross-domain comparability of the
regime decomposition rather than indicating a coverage anomaly. The threshold choice is conservative in our favor: stricter censoring
($k{=}7$) lowers PCS to 0.6692, and looser censoring ($k{=}28$) raises it to 0.7721,
monotonically with the amount of lifecycle structure retained---the
direction predicted by the assortment-churn hypothesis, with all qualitative
comparisons unchanged.

\begin{table}[H]
\centering
\caption{Cross‑domain instantiation of the pathology taxonomy.}
\label{tab:cross_domain}
\scriptsize
\setlength{\tabcolsep}{4pt}
\begin{tabular}{@{}p{2.6cm}p{2.5cm}p{2.5cm}p{2.5cm}p{2.5cm}@{}}
\toprule
\textbf{Pathology} & \textbf{Retail} & \textbf{Energy / Grid} & \textbf{Manufacturing / IoT} & \textbf{Finance} \\
\midrule
Skew & Promo bursts; seasonal peaks & Extreme weather load spikes & Defect clustering; rework spikes & Jump‑driven return asymmetry; earnings announcement effects \\
\addlinespace[3pt]
ZIR & Stockouts; intermittently stocked SKUs & Turbine downtime; zero irradiance & Sensor dropout; line stoppages & Halted trading; zero‑tick intervals in illiquid assets \\
\addlinespace[3pt]
CV & Long‑tail SKU heterogeneity & Ramping‑driven output dispersion across seasons & Batch‑to‑batch variance & Cross‑sectional dispersion of return scales \\
\addlinespace[3pt]
Coupled Distributional Pathologies & Sporadic luxury bursts; Skew‑CV: high‑variance promo series & Zero‑output + sudden load spikes; weather‑driven price spikes & Stoppage + large‑batch orders; high‑variance defect batches & Illiquidity + price jumps; earnings‑driven clustered volatility \\
\addlinespace[3pt]
Coupling with Temporal‑Structure & Pathology + holiday / lifecycle shifts & Pathology + policy / tariff changes & Pathology + line changeovers & Pathology + macro regime transitions \\
\bottomrule
\end{tabular}
\vspace{2pt}
\parbox{\linewidth}{\textbf{\footnotesize Note:} Skewness, ZIR, and high CV are first‑order distributional pathologies; coupled distributional regimes superimpose two or more first‑order pathologies in one marginal; temporal‑structure coupling additionally superimposes regime‑level shifts and is orthogonal to the distributional focus of this paper.}
\end{table}

\begin{table}[H]
\centering
\vspace{-10pt}
\caption{Pathology coverage and industrial validity of mainstream forecasting benchmarks. Higher PCS represents stronger industrial pathological spectrum coverage and better deployment-aligned robustness evaluation capability.}
\label{tab:benchmark_pcs}
\scriptsize
\resizebox{\linewidth}{!}{
\begin{tabular}{lrllccccp{2.9cm}}
\toprule
Dataset & \#Series & Domain \& Granularity & Low Skewness & Low CV & Low ZIR & \textbf{PCS} & Benchmark Grade \\
\midrule
ETT & 28 & Energy (transformer temp), hourly/15-min & 0.5714 & 0.8571 & 0.8571 & \textbf{0.2381} & Unqualified; basic temporal verification only \\
M4-Daily & 4{,}227 & Mixed (macro$\sim$micro), daily & 0.863 & 0.9970 & 1.000 & \textbf{0.046} & Insufficient pathological spectrum \\
M4-Hourly & 414 & Mixed (macro$\sim$micro), hourly & 0.865 & 0.993 & 1.0 & \textbf{0.048} & Insufficient pathological spectrum \\
M4-Monthly & 48{,}000 & Mixed (macro$\sim$micro), monthly & 0.845 & 0.999 & 1.0 & \textbf{0.052} & Insufficient pathological spectrum \\
M4-Quarterly & 24{,}000 & Mixed (macro$\sim$micro), quarterly & 0.849 & 0.999 & 1.0 & \textbf{0.051} & Insufficient pathological spectrum \\
M4-Weekly & 359 & Mixed (macro$\sim$micro), weekly & 0.652 & 0.997 & 1.0 & \textbf{0.117} & Insufficient pathological spectrum \\
M4-Yearly & 23{,}000 & Mixed (macro$\sim$micro), yearly & 0.864 & 0.999 & 1.0 & \textbf{0.045} & Insufficient pathological spectrum \\
M5 & 30{,}490 & Retail (Walmart sales), daily & 0.0513 & 0.5478 & 0.1248 & \textbf{0.7587} & Adequate coverage; adopted as external validation (App.~E) \\
RetailShiftBench & 31{,}213 & Retail (SKU demand), daily & 0.2573 & 0.3925 & 0.192 & \textbf{0.7194} & Fully qualified; full-spectrum stress test \\
\bottomrule
\end{tabular}
}
\vspace{4pt}
\parbox{\linewidth}{\scriptsize \textit{Note}: PCS thresholds are calibrated on RetailShiftBench retail data; for M4 and M5, thresholds are applied as-is with the zero-run censoring above, and domain-adaptive percentile recalibration (\S~\ref{sec:taxonomy}) is recommended for auditing other domains.}
\vspace{-10pt}
\end{table}

\paragraph{Threshold recalibration.} The \S~\ref{sec:taxonomy_foundational} thresholds are retail-calibrated. Domain-specific adoption only requires monotonic recalibration: (i) skewness---use the 90th percentile of the domain's empirical $\gamma_1$; (ii) ZIR---define tiers by the domain's $Q_1/Q_3$ of ZIR; (iii) CV---anchor ``high'' at $1.5\times$ domain median. Coupled regimes depend only on threshold monotonicity, so their identification is invariant to such recalibration.

\paragraph{Statistical-level transferability.}
PCS operates on the domain-independent statistical level, making it
directly transferable across domains: the threshold recalibration
described above suffices, with no domain-specific objective required.
Thus, a low-PCS dataset in any domain is equally unqualified for
industrial stress-testing, confirming that the auditing paradigm
generalizes beyond retail.

\section{Decomposition Response Analysis}
\label{app:decomp_response}
This appendix expands the two main tables cell by cell
(\S~\ref{app:walkthrough}), isolates the attribution of decomposition
gains (\S~\ref{app:quantile_spectra}), and traces the intervention into the
quantile-calibration lens (Fig.~\ref{fig:quantile_calibration}). All
quantitative statements follow the protocol of
\S~\ref{sec:metric_agnostic_protocol}: identical hyperparameters across
losses, Huber threshold fixed at the training-window residual quantile,
and paired train/test splits shared by the global and decomposed
evaluations. Unless stated otherwise, all experiments in this appendix
are conducted on the main RetailShiftBench benchmark
(\S~\ref{sec:experimental_setup}).

\subsection{Reading Tables~\ref{tab:regime_perf}--\ref{tab:decompose_effect}}
\label{app:walkthrough}

Three structural readings organize Table~\ref{tab:regime_perf}.
\emph{(i) Masking.} The \texttt{All} row certifies success within tight
tolerances, while pathological regimes exhibit systematic errors an order
of magnitude larger---visible cell by cell in the table. The aggregate
passes precisely because each pathological subgroup is too small to move
the pooled average: the \texttt{All} row does not summarize the regime rows but dilutes them: no
pathological subgroup is large enough to move the pooled average---the
robustness illusion.

\emph{(ii) Directional fingerprints.} Mean-targeting losses carry
\emph{positive} bias on sparse regimes, median-type losses exhibit strongly
negative bias, and quantile heads exhibit step behavior; each direction
matches the corresponding row of Table~\ref{tab:loss_misspecification},
and the \emph{magnitudes} are exactly the table entries (we avoid
restating them here to prevent transcription drift). The signs, not the
magnitudes, are the fingerprint: a loss family is identified by the
\emph{direction} in which each pathology pushes its bias, and the four
regimes form a diagnostic set from which the loss prior can be decoded.

(iii) Scale-protocol robustness. The regime-dependent bias structure
replicates across XGBoost (raw magnitudes) and TFT (per-series-normalized
internally), with one systematic exception. Median-type losses (MAE,
quantile, Huber) preserve their sign fingerprint---negative bias on
pathological regimes---under both preprocessing pipelines; mean-type
losses (MSE, Tweedie, Poisson) show preprocessing-dependent signs,
flipping from positive under raw magnitudes to negative under internal
normalization on high-CV and zero-inflated cells (Tab.~\ref{tab:regime_perf}).
What is invariant across protocols is not the sign but the structure:
RBV remains elevated under both pipelines for every loss, so the
regime-level misalignment survives normalization even where its sign is
preprocessing-dependent. Normalization changes the units of the residual
and, for mean-type losses, the direction of the dilution response, but it
does not remove the mismatch---supporting that the bias profile is an
estimand signature rather than a scale artifact. We emphasize this is
structural consistency across protocols, not protocol invariance:
magnitudes are not comparable across backbones, and quantitative
comparisons remain within-functional.

Table~\ref{tab:decompose_effect} splits responses to decomposition by
estimand family. Mean-targeting losses (MSE, Tweedie) improve on all
three instruments, because decomposition removes the mixture compromise
without touching their estimand. Median-type losses deteriorate in bias
structure while holding accuracy: with per-regime optimization the
median estimand no longer benefits from pooled smoothing, exposing
estimand-level collapse rather than fitting error. Three cells warrant
individual notes. \emph{Huber} worsens across metrics at its
data-driven threshold and remains in the worsen family across a 40-fold
$\delta$ range: its effective estimand is a rational function of the
sparsity mass ($\propto (1-p_0)\delta/p_0$), so per-regime cells with
different $p_0$ induce different effective estimands under one fixed
$\delta$---the hybrid is the most variance-fragile member of the zoo
under small cells, and no single $\delta$ is simultaneously calibrated
for all regimes. \emph{Poisson} dissociates the three instruments
entirely---structure and squared error improve while the linear metric
degrades---because exponential curvature redistributes predictive mass
in a way that linear-in-residual instruments cannot register: it is the
cleanest exhibit that linear metrics are blind to redistribution by
curvature, not to absence of effect. \emph{MAE under TFT} confirms the
protocol-level reading (iii) above: the fingerprint direction survives
the architecture change even where magnitudes shift.

Finally, decomposed mean-family configurations converge toward the Rule
heuristic's bias flatness while retaining far better accuracy,
populating the interior of the aggregation trade-off
(Fig.~\ref{fig:agg_tradeoff}) with measured anchors---each such
configuration is an operating point between pooled compromise and
per-series purity that was previously unmeasured.

$\uparrow/\downarrow$ mark directionally consistent changes across all
random seeds (3 seeds); paired-bootstrap intervals are reported as
consistency checks rather than confirmatory tests.

\subsection{Attribution Controls and Quantile Calibration Spectra}
\label{app:quantile_spectra}

Decomposition gains could in principle arise from data volume, cell
homogeneity, or label informativeness. We separate these with two
controls and inspect the intervention through two complementary
lenses---scalar bias structure and quantile calibration.
\emph{(a) Data-matched global}: a single pooled model trained on the
concatenation of all cell data with matched optimization budget.
\emph{(b) Random-label cells (R2)}: a permutation (sampled without
replacement) of series into cells matched in number and size to the
pathology decomposition, over 3 shuffle seeds---reproducing the
granularity of the decomposition while destroying its semantics.

\begin{figure}[H]
\centering
\includegraphics[width=\linewidth]{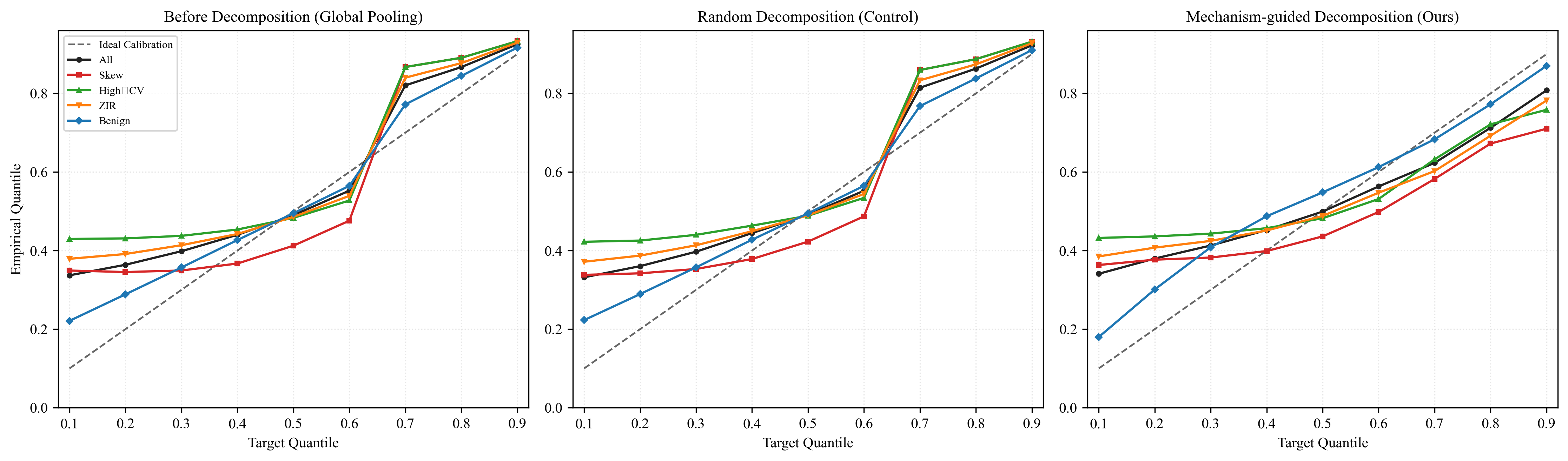}
\caption{Quantile calibration as an attribution test (XGBoost,
single-$\tau$ pinball objectives, tree depth fixed across $\tau$).
Coverage is computed on positive-demand days with the mid convention
$P(\hat y > y) + \tfrac{1}{2} P(\hat y = y)$; the dashed line marks
ideal calibration. \emph{Left}: global pooling glues pathological
regimes below the diagonal while benign stays over-covered.
\emph{Middle}: random-label decomposition with cells matched in count
and size to the pathology partition (permutation without replacement)
reproduces the glued pattern---granularity alone does not release the
curves. \emph{Right}: pathology-informed decomposition releases the
bundle toward the diagonal, with the residual gap at $\tau{=}0.9$
marking the variance wall. Near-zero low-$\tau$ coverage on
pathological regimes reflects the estimand floor of zero-inflated
marginals (\S~\ref{sec:rbv_def}), not additional miscalibration.}
\label{fig:quantile_calibration}
\end{figure}

Two readings organize the figure. In the high-$\tau$ band
($\tau \geq 0.5$), pooled training fans the regimes apart, the
random-label control (b) reproduces this
fan almost exactly, and only pathology-informed decomposition releases the
bundle toward the diagonal---label informativeness, not partition
granularity, is the operative variable. In the low-$\tau$ band,
pathological coverage sits at the estimand floor of zero-inflated
marginals (the calibrated predictor is itself zero); this reflects the
plateau degeneracy of \S~\ref{sec:rbv_def} rather than additional
miscalibration, and is annotated in the figure but not interpreted as a
gap. The \emph{false-zero} rate $P(\hat{y} = 0 \mid y > 0)$, the direct
collapse diagnostic for this floor, is overlaid per regime in the same
panel: decomposition collapses it only where pooling inflates it, i.e.,
exactly in the ZIR-dominated band.

\begin{table}[H]
\centering
\scriptsize
\caption{Loss-function trajectory grid under grouping-scheme interventions (XGB).}
\label{tab:tab_loss_trajectory_grid}
\vspace{-5pt}
\setlength{\tabcolsep}{6pt}
\setlength{\belowcaptionskip}{2pt}
\renewcommand{\arraystretch}{1.15}
\resizebox{\linewidth}{!}{
\begin{tabular}{l c c}
\hline
& \multicolumn{2}{c}{\textbf{non-linear loss family}} \\
\cline{2-3}
Metric & MSE & Tweedie \\
\hline
RBV   & $0.243 \rightarrow 0.26 \rightarrow 0.041 \downarrow$  & $0.163 \rightarrow 0.151 \rightarrow 0.014\downarrow$ \\
WMAPE & $0.4679 \rightarrow 0.4681 \rightarrow 0.462 \downarrow$  & $0.4593 \rightarrow 0.4606 \rightarrow 0.4539\downarrow$ \\
MSE & $11.909 \rightarrow 12.615 \rightarrow 11.068 \downarrow$  & $11.992 \rightarrow 12.678 \rightarrow 11.465\downarrow$ \\
\hline
& \multicolumn{2}{c}{\textbf{linear loss family}} \\
\cline{2-3}
Metric & MAE & Huber \\
\hline
RBV   & $0.675 \rightarrow 0.658 \rightarrow 0.899\uparrow$   & $0.077 \rightarrow 0.24 \rightarrow 0.262\uparrow$ \\
WMAPE & $0.4270 \rightarrow 0.4316 \rightarrow 0.4255\downarrow$ & $0.4670 \rightarrow 0.4669 \rightarrow 0.4765\uparrow$ \\
MSE & $12.057 \rightarrow 12.046  \rightarrow 11.524\downarrow$  & $13.575 \rightarrow 13.048 \rightarrow 20.176 \uparrow$\\
\hline
\end{tabular}
}
\vspace{3pt}
\parbox{\linewidth}{\scriptsize
Notes: Each cell reports ``All $\rightarrow$ Random $\rightarrow$ Regime-wise''.
All: global pooled training; Random: mean over three shuffle seeds;
Regime-wise: pathology-informed per-regime training.
All and Regime-wise values reproduced from Tab.~\ref{tab:decompose_effect} for attribution contrast:
All$\rightarrow$Random isolates granularity effects; Random$\rightarrow$Regime-wise isolates label informativeness at matched granularity;
All$\rightarrow$Regime-wise is total intervention effect.
$\downarrow$: improvement, $\uparrow$: degradation. Losses grouped by residual curvature (\S~\ref{sec:discussion}).
TFT results follow the same pattern. Rule reference values: Tab.~\ref{tab:decompose_effect}.
}
\end{table}

\subsection{Capacity Scaling Cannot Substitute for Regime Decomposition}
\label{app:capacity}

\begin{figure}[H]
\centering
\includegraphics[width=\textwidth]{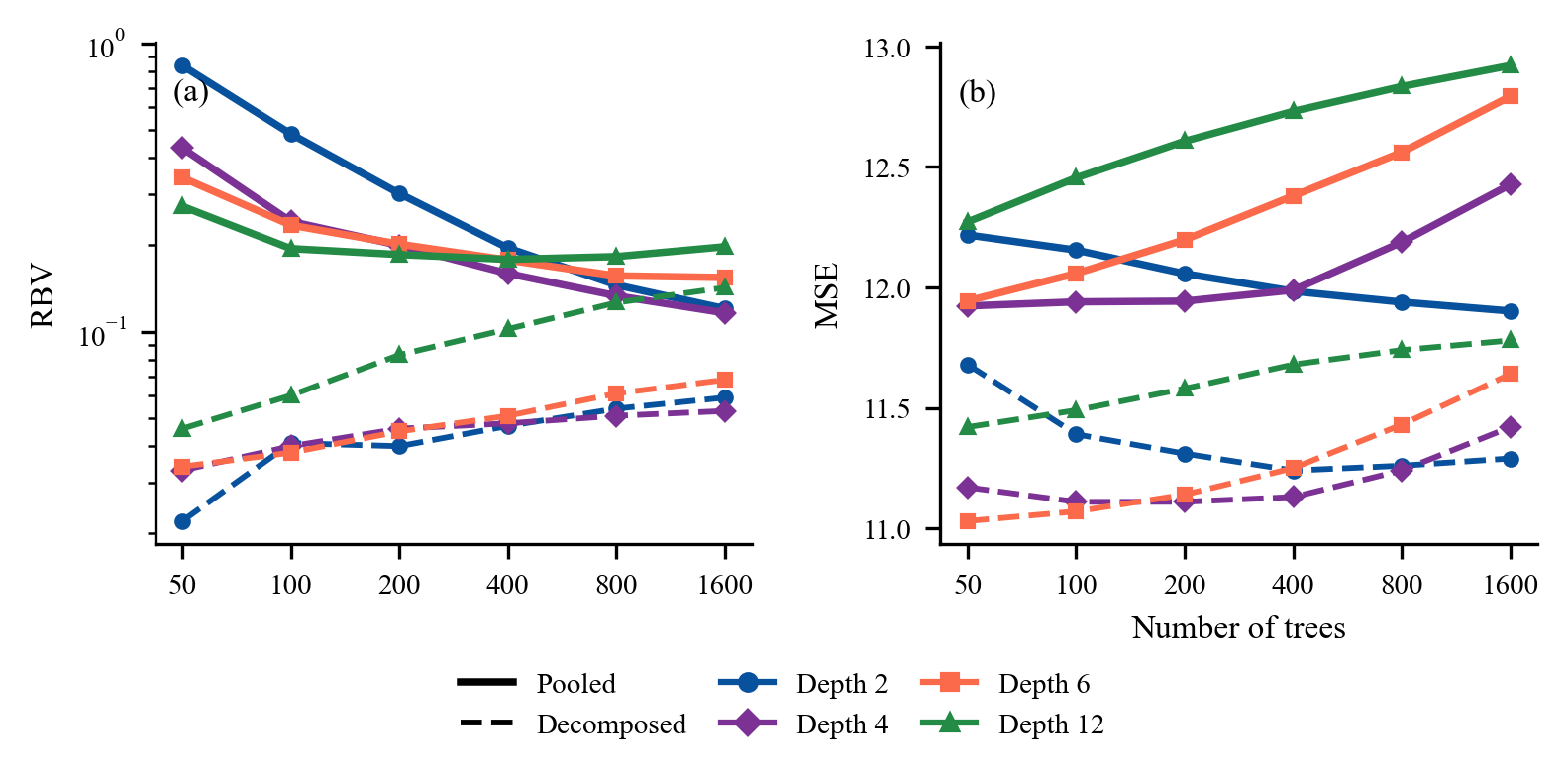}
\caption{Capacity scaling (solid) vs.\ regime decomposition (dashed) under
the MSE loss. (a) RBV (log scale): pooled depth-2 RBV improves $7\times$
($0.837 \to 0.120$) with $32\times$ the trees, and the best pooled
configuration overall (1600 trees, depth 4) saturates at a $0.116$
ceiling---decomposition (dashed) holds RBV $\leq 0.06$ across all but the
most extreme configurations (within-cell overfitting at 1600 trees,
depth 12). (b) MSE: decomposition improves aggregate accuracy as
well. WMAPE varies within $0.46$--$0.50$ across all conditions and is
omitted.}
\label{fig:capacity_mse}
\end{figure}

\begin{figure}[H]
\centering
\includegraphics[width=\textwidth]{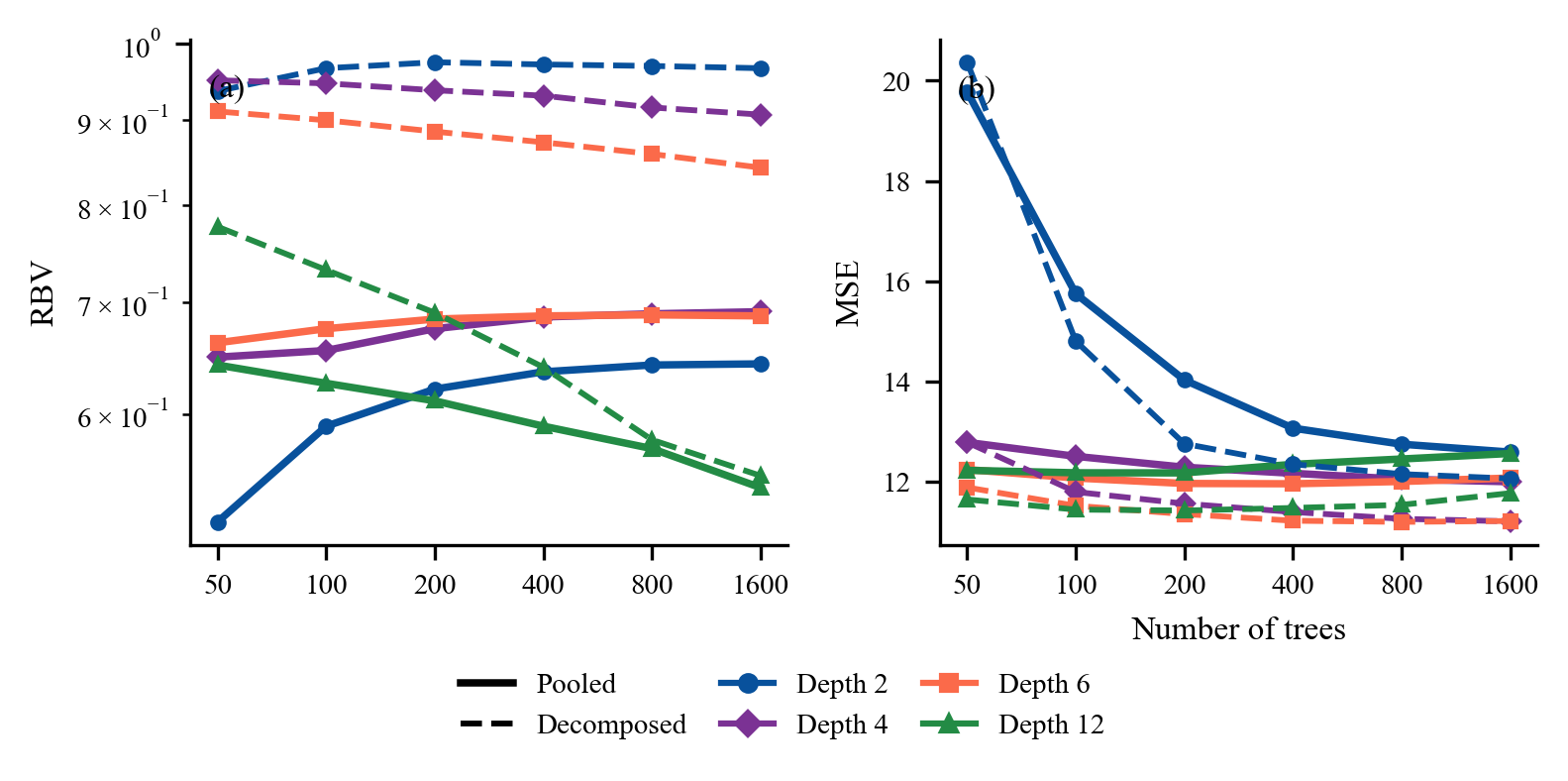}
\caption{Same setup as Fig.~\ref{fig:capacity_mse} under the MAE loss.
Decomposition \emph{increases} RBV (dashed curves lie at or above solid ones
in (a)): routing removes the benign mass that dilutes MAE's within-cell
estimand mismatch, exposing the mean--median gap rather than repairing
it---whereas aggregate error (b) still improves.}
\label{fig:capacity_mae}
\end{figure}

A natural objection to the decomposition results of Tab.~\ref{tab:decompose_effect} is
that pooling's failures might simply reflect under-capacity: a sufficiently
expressive model could, in principle, fit regime-conditional structure
implicitly, recovering what explicit routing provides. We test this objection
directly by scaling model capacity under a fixed data budget, with and
without decomposition.

\paragraph{Setup.} We train gradient-boosted trees (XGBoost) on all
RetailShiftBench series (pooled) and, separately, one model per regime cell
(decomposed, \S~\ref{sec:metric_agnostic_protocol}), sweeping trees
$n_{\mathrm{est}} \in \{50,100,200,400,800,1600\}$ and depth
$d \in \{2,4,6,8,10,12\}$ under MSE and MAE losses. All arms share the identical
train/test split; runs differ only in model stochasticity (3 seeds). No early
stopping is used, so that capacity effects are not confounded with
fit-versus-validation trade-offs.

\paragraph{Three observations.}
(i) \emph{Saturation.} Under MSE, pooled RBV improves with capacity and
saturates: depth-2 RBV improves $7\times$ ($0.837 \to 0.120$) with
$32\times$ the trees, and the best pooled configuration (1600 trees,
depth 4) stalls at a ceiling of $0.116$, with no configuration falling below
it---a residual plateau, not convergence to zero
(Fig.~\ref{fig:capacity_mse}a); at depth $\geq 8$ RBV already rebounds
while capacity keeps growing.
(ii) \emph{Decoupling.} The configuration that minimizes aggregate error
(MSE) is not the configuration that minimizes RBV: MSE favors depth 2
($11.90$), whereas RBV favors depth 4 ($0.116$), and RBV \emph{rebounds} at
depth $\geq 6$ while MSE degrades only mildly ($11.90 \to 12.92$).
Aggregate metrics therefore do not track, and cannot proxy for, regime-level
bias structure.
(iii) \emph{Breakthrough.} Decomposition collapses RBV at every matched
configuration: across the full grid, decomposed RBV stays within
$0.022$--$0.142$ and never exceeds the pooled value at the same
$(n_{\mathrm{est}}, d)$; at the smallest configuration it is
$0.022$--$0.046$, i.e., $32\times$ fewer trees than the pooled optimum, and
it simultaneously improves MSE
(Fig.~\ref{fig:capacity_mse}b). There is no accuracy--robustness trade-off.

\paragraph{Interpretation.} These patterns are the empirical signature of
the separation in Prop.~\ref{prop:decomposition}. Capacity acts on the fitting
channel $e_t$: as trees and depth grow, the pooled model fits the training
distribution ever more closely, and RBV falls toward the pooled compromise
$b_L$, where it stalls. The saturated residual is the aggregation-induced compromise that
increased capacity cannot express away, since it is not an
approximation error. Decomposition acts on the other channel, removing the pooling
compromise by construction---under a mean target $b_L\equiv0$, so for
MSE the saturated residual is entirely this compromise; the decomposed
arm's residual RBV is small, and its mild growth at extreme depth
reflects within-cell overfitting rather than any re-emergence of pooling
bias. Under MAE the asymmetry confirms the diagnosis:
decomposition \emph{increases} RBV (dashed curves lie above pooled ones at
every matched configuration, Fig.~\ref{fig:capacity_mae}a), because MAE's
mismatch is located \emph{within} cells---the mean--median gap of
Table~\ref{tab:loss_misspecification}---which pooled training dilutes and
routing exposes. Routing and loss choice address orthogonal exposures;
neither substitutes for the other.
 
\section{M5 Dataset and External Validation}
\label{app:m5_external}

\subsection{M5 Experiment Setup}
\label{app:m5_exp_setup}
We replicate the regime-diagnosis pipeline on the M5 forecasting
competition dataset~\citep{makridakis2022m5_bak}, a large-scale retail
demand panel that differs from our main benchmark in scale, granularity,
and demand sparsity. Note that our M5 protocol is designed for controlled
regime-diagnosis comparison, not for leaderboard benchmarking: we
construct 30,490 SKU$\times$store series from the final 120 days
($d_{1822}$--$d_{1941}$, 2016-01-24 to 2016-05-22) of the five-year
training partition and, following the single-origin protocol, use the
first 113 days as the training window, forecasting horizons $1$--$7$
from the single origin $t_{\mathrm{idx}}=112$ and evaluating against
realized demand on days 114--120; these choices are not comparable to
the official M5 competition setup, and our M5 numbers should not be read
against competition rankings or published M5 leaderboards. Of the 30,490
constructed series, 1,995 fail the 28-day minimum-history screen for
regime statistics (ZIR, skewness, CV), computed exclusively within the
training window; a further 1,749 of the remaining 28,495 (6.1\%) lack
usable evaluation signal at the forecast origin (zero demand in the
evaluation window or insufficient history for the deep backbones),
leaving 26,746 scored series. Pathological flags use the thresholds of
Section~\ref{sec:taxonomy_foundational} with the not-low binarization,
yielding overlapping coupled cells: most valid series carry at least one
pathological flag, the CV condition is almost always accompanied by the
other two, and only about 10\% are Benign. We train the same XGBoost
backbones under the same six losses plus a pinball-loss scan over
$\tau\in\{0.1,\dots,0.9\}$.

\paragraph{Censoring scope and sensitivity.}
The zero-run censoring convention of App.~\ref{app:cross_domain}
is applied at both levels: the full-panel PCS audit and the 120-day
evaluation window constructed here. Censoring masks individual time
steps---days inside a zero run longer than $k{=}14$ are set to missing
and excluded from pathology statistics and from training---while all
\num{30490} series are retained. The convention is far from inert at the window scale: 17.3\% of series
contain at least one censored edge run (leading: 13.5\%, trailing: 5.5\%,
the shares overlapping as one series may carry both), confirming that
assortment churn also reaches the evaluation horizon and must be censored
rather than absorbed into ZIR and skewness. All results in this appendix are computed on the
censored window.

\paragraph{Regime-decomposed replication.}

\begin{table}[H]
\centering
\scriptsize
\caption{Loss-function trajectory grid under grouping-scheme interventions on M5 (XGB, fixed-origin protocol; mean over three seeds 42, 41, 39).}
\label{tab:tab_m5_loss_trajectory}
\vspace{-5pt}
\setlength{\tabcolsep}{6pt}
\setlength{\belowcaptionskip}{2pt}
\renewcommand{\arraystretch}{1.15}
\resizebox{\linewidth}{!}{
\begin{tabular}{l c c c}
\hline
& \multicolumn{3}{c}{\textbf{non-linear loss family}} \\
\cline{2-4}
Metric & MSE & Tweedie & Poisson \\
\hline
RBV   & $0.422 \rightarrow 0.436{\scriptscriptstyle\pm0.0071} \rightarrow 0.036\downarrow$ & $0.377 \rightarrow 0.359{\scriptscriptstyle\pm0.0019} \rightarrow 0.062\downarrow$ & $0.476 \rightarrow 0.487{\scriptscriptstyle\pm0.0018} \rightarrow 0.044\downarrow$ \\
WMAPE & $0.7597 \rightarrow 0.7612{\scriptscriptstyle\pm0.0004} \rightarrow 0.7337\downarrow$ & $0.7553 \rightarrow 0.7526{\scriptscriptstyle\pm0.0005} \rightarrow 0.7259\downarrow$ & $0.7612 \rightarrow 0.7612{\scriptscriptstyle\pm0.0001} \rightarrow 0.7326\downarrow$ \\
MSE & $4.784 \rightarrow 4.745{\scriptscriptstyle\pm0.020} \rightarrow 4.607\downarrow$ & $4.758 \rightarrow 4.699{\scriptscriptstyle\pm0.020} \rightarrow 4.493\downarrow$ & $4.753 \rightarrow 4.690{\scriptscriptstyle\pm0.025} \rightarrow 4.505\downarrow$ \\
\hline
& \multicolumn{3}{c}{\textbf{linear loss family}} \\
\cline{2-4}
Metric & MAE & QL60 & Huber \\
\hline
RBV   & $0.408 \rightarrow 0.397{\scriptscriptstyle\pm0.0021} \rightarrow 0.804\uparrow$ & $0.105 \rightarrow 0.072{\scriptscriptstyle\pm0.0022} \rightarrow 0.871\uparrow$ & $0.241 \rightarrow 0.251{\scriptscriptstyle\pm0.00026} \rightarrow 0.367\uparrow$ \\
WMAPE & $0.6896 \rightarrow 0.6895{\scriptscriptstyle\pm0.0003} \rightarrow 0.6590\downarrow$ & $0.7301 \rightarrow 0.7318{\scriptscriptstyle\pm0.0003} \rightarrow 0.6958\downarrow$ & $0.7272 \rightarrow 0.7283{\scriptscriptstyle\pm0.00003} \rightarrow 0.6963\downarrow$ \\
MSE & $5.051 \rightarrow 4.943{\scriptscriptstyle\pm0.022} \rightarrow 4.705\downarrow$ & $4.969 \rightarrow 4.869{\scriptscriptstyle\pm0.010} \rightarrow 4.686\downarrow$ & $4.831 \rightarrow 4.778{\scriptscriptstyle\pm0.004} \rightarrow 4.603\downarrow$ \\
\hline
\end{tabular}
}
\vspace{3pt}
\parbox{\linewidth}{\scriptsize
Notes: Each cell reports ``All $\rightarrow$ Random $\rightarrow$ Regime-wise''.
All: global pooled training; Random: random split with cells matched
in count and size to the regime partition, mean$\pm$std over three shuffle
seeds (42, 41, 39); Regime-wise: per-cell training on the three-tier
regime partition (skew--CV--ZIR) computed from the training window only.
Protocol: 113 training days per series, single forecast origin at the last training day,
horizons $h{=}1,\dots,7$ pooled (26{,}746 scorable series; 79 series in sub-guard cells excluded).
All$\rightarrow$Random isolates granularity effects; Random$\rightarrow$Regime-wise isolates
label informativeness at matched granularity.
$\downarrow$: improvement, $\uparrow$: degradation.
MSE is reported in raw demand units; WMAPE and RBV are scale-free and serve as the
cross-regime instruments.
Unlike RetailShiftBench (Tab.~\ref{tab:tab_loss_trajectory_grid}), decomposition improves
WMAPE and MSE for \emph{all} losses while degrading RBV for the linear family---the
accuracy--robustness decoupling on real data. Grouping-induced RBV increases for
quantile-type losses track the model-independent intrinsic floor of
cell-level $(Q_\tau-\mu)/\mu$ dispersion (see \S\ref{app:rbv_floor}).
}
\end{table}

\begin{figure}[t]
\centering
\includegraphics[width=\textwidth]{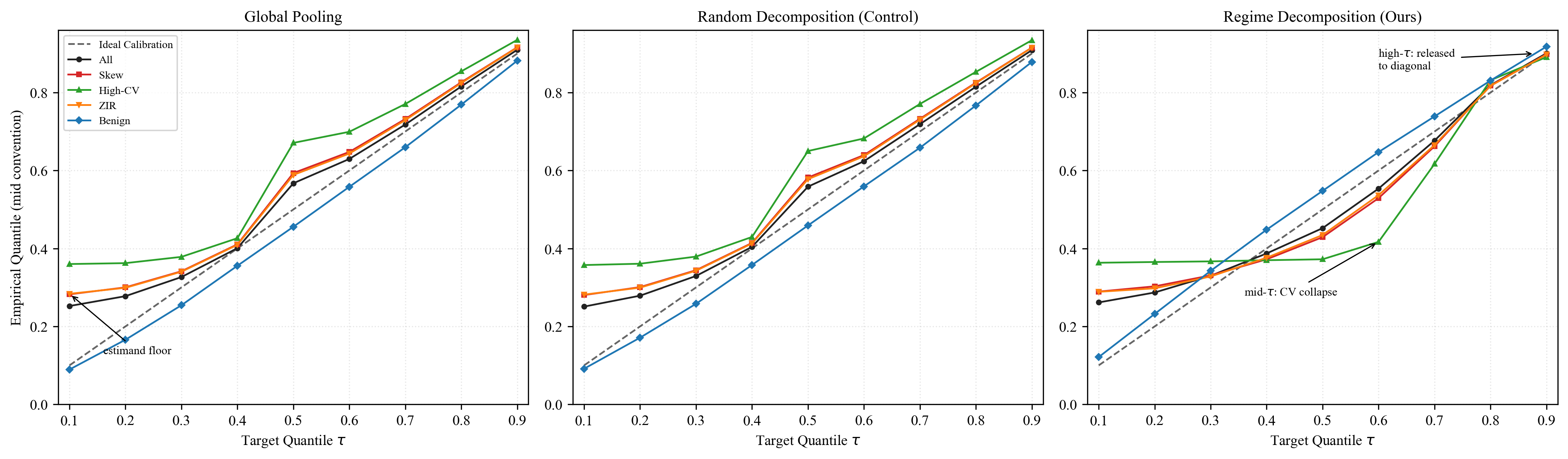}
\caption{Quantile-calibration attribution on M5
(mean over three seeds 42, 41, 39;$\tau\in\{0.1,\dots,0.9\}$; $\tau{=}0.5$ shown via MAE, since
$\mathrm{pinball}(0.5)\equiv\mathrm{MAE}$). Exceedance uses the mid
convention $\Pr[\hat y>y]+\tfrac12\Pr[\hat y=y]$.
\textbf{Left:} global pooling --- the aggregate curve hugs the diagonal
while regime curves fan out.
\textbf{Center:} random decomposition (size-matched control) --- 
reproduces the pooled panel, isolating label informativeness.
\textbf{Right:} regime decomposition --- high-$\tau$ curves are
released to the diagonal while the High-CV collapse at mid $\tau$ is
exposed. Low-$\tau$ elevations reflect the zero-inflation estimand
floor ($\mathrm{zir}/2$), not miscalibration.}
\label{fig:fig_m5_quantile_calibration}
\end{figure}

Table~\ref{tab:tab_m5_loss_trajectory} reports the
All$\,\to\,$Random$\,\to\,$Regime-wise trajectory on M5. The two
attribution steps of Section~\ref{app:quantile_spectra} replicate cleanly:
the size-matched random decomposition is inert for all thirteen losses
($|\Delta\mathrm{RBV}|\le 0.04$), ruling out the granularity confound,
whereas mechanism-guided decomposition reshapes the loss landscape in a
loss-family-dependent manner --- non-linear losses (MSE, Tweedie,
Poisson) see RBV collapse to near zero, while linear losses (MAE,
QL60, Huber) show higher RBV that tracks the model-independent
intrinsic floor of cell-level $(Q_\tau-\mu)/\mu$ dispersion (a perfect
$\tau$-quantile forecaster would itself exhibit RBV of $0.76$ and
$0.77$ under regime grouping for MAE and QL60, respectively), with
per-cell exceed rates simultaneously moving toward $\tau$
(Fig.~\ref{fig:fig_m5_quantile_calibration}); i.e., the increase
exposes objective-induced gaps rather than worsening fit --- and
pinball losses flip sign at $\tau\!\approx\!0.6$--$0.7$. This mirrors
the main benchmark and provides a quantitative cross-dataset match,
e.g., the High-CV mean-relative bias of MAE moves $-0.60 \to -0.95$
on M5 versus $-0.58 \to -0.98$ on the main benchmark.

Figure~\ref{fig:fig_m5_quantile_calibration} replicates the
quantile-calibration attribution of Figure~\ref{fig:quantile_calibration} on M5. Left:
global pooling glues the aggregate curve to the diagonal while regime
curves fan out (Benign below, High-CV above); the random control
reproduces the pooled panel; and regime decomposition releases high-$\tau$
curves to the diagonal while exposing a mid-$\tau$ High-CV collapse
(e.g., $\mathrm{Exceed}=0.417$ at $\tau=0.6$). We note that on
zero-inflated demand the low-$\tau$ elevation of the aggregate curve is
largely an \emph{estimand floor} rather than miscalibration: a perfectly
calibrated predictor equals zero whenever $\tau$ falls below the
zero-mass, so its exceedance is $\mathrm{zir}/2$ under the mid
convention. The random-decomposition column of
Table~\ref{tab:tab_m5_loss_trajectory} averages over three shuffle
seeds (42, 41, and 39); mean$\pm$std over seeds is reported in the
table notes.

\paragraph{Quantile objectives and the intrinsic bias floor.}
For quantile (and median) losses, $B_r$ compares a $\tau$-quantile
forecast against the conditional mean, which is nonzero by construction;
the gap $(Q_\tau-\mu)/\mu$ is large in zero-inflated and high-variance
cells. Empirically this imposes a model-independent floor on RBV:
computing $B_r$ from the empirical training distribution of each cell
(a perfect $\tau$-quantile forecaster) yields RBV$^*$ that tracks the
regime-grouped observed RBV closely (MAE: $0.80$ vs.\ $0.76$; ql60:
$0.87$ vs.\ $0.77$; ql90: $1.08$ vs.\ $1.22$), with the U-shaped
dependence on $\tau$ reproduced. Pooled training can instead push RBV below the floor (ql60: 0.11 $\ll$
0.77), not by fitting better but by imposing a single compromise quantile
under which per-cell exceedances cluster (0.56--0.70 at nominal 0.6),
yielding low dispersion of the mean-relative bias across cells even
though every cell remains miscalibrated; under regime grouping the
exceedances return to their nominal level (ql90: 0.92/0.90/0.90/0.89
against 0.9). For quantile losses RBV must therefore be read
jointly with calibration (Fig.~\ref{fig:fig_m5_quantile_calibration});
the grouping-induced RBV increase exposes objective-induced gaps rather
than worsening misspecification.
\label{app:rbv_floor}

\begin{table}[H]
\vspace{-5pt}
\caption{\textbf{Regime decomposition on the M5 deep backbone (NHiTS).}
Each cell reports All $\to$ Random $\to$ Regime-wise training.
All: global pooled training; Random: random split with cells matched in count
and size to the pathology partition (mean over three shuffle seeds);
Regime-wise: per-cell training on the skew--CV--ZIR partition. Single seed
(42) for All and Regime-wise.}
\label{tab:deep_grouping_summary}
\vspace{-5pt}
\centering
\small
\begin{tabular}{l ccc ccc}
\toprule
& \multicolumn{3}{c}{RBV} & \multicolumn{3}{c}{WMAPE} \\
\cmidrule(lr){2-4}\cmidrule(lr){5-7}
Loss & All & Random & Regime & All & Random & Regime \\
\midrule
MSE     & 0.473 & 0.472 & 0.400 $\downarrow$ & 0.7526 & 0.7381 & 0.6764 $\downarrow$ \\
MAE     & 0.428 & 0.454 & 0.416 $\downarrow$ & 0.6660 & 0.6918 & 0.6625 $\downarrow$ \\
QL60    & 0.420 & 0.437 & 0.480 $\uparrow$   & 0.6791 & 0.7096 & 0.6910 $\uparrow$ \\
Huber   & 0.397 & 0.390 & 0.369 $\downarrow$ & 0.6911 & 0.6887 & 0.6659 $\downarrow$ \\
Tweedie & 0.500 & 0.469 & 0.421 $\downarrow$ & 0.7818 & 0.7236 & 0.6810 $\downarrow$ \\
Poisson & 0.470 & 0.453 & 0.408 $\downarrow$ & 0.7057 & 0.7017 & 0.6707 $\downarrow$ \\
\bottomrule
\end{tabular}
\vspace{-3pt}
\begin{minipage}{0.95\linewidth}
\vspace{3pt}
\footnotesize
Notes: NHiTS backbones on the M5 single-origin protocol
($n = 173{,}152 = 24{,}736$ encoder-eligible series $\times$ $7$ horizons;
series lacking encoder history at the origin are excluded, cf.\
Tab.~\ref{tab:tab_m5_loss_trajectory}). Chronos-Bolt, being zero-shot
with sliding-window context and no fixed encoder length, instead scores
the full $26{,}746$-series scored panel (short series are left-padded
rather than excluded), so its eligible set strictly contains the NHiTS
encoder-eligible subset and cross-model comparisons read at the ranking
level. $\downarrow$: improvement, $\uparrow$: degradation. Regime-wise
splitting improves WMAPE sharply (optimization-type gain) but moves RBV
only modestly, unlike the converged XGBoost backbone
(Tab.~\ref{tab:tab_m5_loss_trajectory}): the residual bias is a CV-cell
zero-collapse (predicted within-cell CV $0.06$--$0.11$ vs.\ realized
$0.47$) that per-cell training cannot repair. The QL60 RBV increase
tracks the model-independent estimand floor (\S\ref{app:rbv_floor}),
mirroring the XGBoost quantile family.
\end{minipage}
\vspace{-5pt}
\end{table}

\paragraph{Deep-backbone replication.}
To test whether the characterized anatomy is specific to converged
gradient-boosted training, we replicate the grouping intervention on an
NHiTS backbone, which operates in an under-trained regime on this panel.
Tab.~\ref{tab:deep_grouping_summary} reports the same
All$\to$Random$\to$Regime-wise trajectory. The diagnostic pattern
inverts: Regime-wise training improves WMAPE sharply
(optimization-type gain) yet moves RBV only modestly, and the residual
bias concentrates in the CV cell, where predicted within-cell dispersion
collapses ($0.06$--$0.11$ predicted vs.\ $0.47$ realized)---a failure
mode that per-cell training cannot repair. The deep backbone thus
exhibits optimization-type failure where XGB exhibits bias-type failure;
the diagnosis separates the two, and only the bias-type component is
repairable by re-targeting.

\subsection{OR Baselines and Oracle Reference Predictors}
\label{app:or_m5}

\paragraph{OR baselines and oracle reference predictors.}
Table~\ref{app:tab_or_oracle_m5} reports Croston, SBA, TSB, and
TSB-HB alongside two oracle reference predictors on M5. Three
observations. (i) All four OR methods track the oracle mean (excess
bias within $\pm0.08$ in every cell, $\pm0.04$ for SBA and TSB-HB):
their low RBV (0.019--0.028) reflects genuine per-series mean-rate
calibration, not insensitivity---an order of magnitude below every
trained loss in Tab.~\ref{tab:tab_m5_loss_trajectory}. (ii) The
oracle median---perfectly calibrated for its own functional---under-forecasts
by 46--94\% on pathological regimes, demonstrating that estimand choice,
not model quality, dictates the mismatch profile. (iii) The oracle
median attains the best WMAPE (0.668) with the worst bias allocation
(RBV $0.695$), a direct exhibit of our thesis that aggregate accuracy
certifies neither the presence nor direction of mismatch. The cost of
OR calibration is efficiency: their WMAPE (0.725--0.739) rails pooled
XGB-MAE. This contrast is a calibration reference, not an accuracy ranking.

\begin{table}[t]
\centering
\caption{Intermittent-demand baselines and oracle references on the M5
single-origin protocol. Bias $B_r$ follows the paper's sample-pooled
sum-ratio convention, and RBV is the unweighted $\ell_2$ norm of the
centered bias vector (identical to Tab.~\ref{tab:tab_loss_trajectory_grid}).
Excess bias is relative to oracle-mean. Croston's uniformly positive
bias is its known over-prediction; SBA's damping factor largely removes
it. Oracle-median achieves the best WMAPE yet the largest regime bias
variation: accuracy under MAE does not imply calibration for the
pooled-mean estimand.}
\label{app:tab_or_oracle_m5}
\small
\begin{tabular}{lrrrrrrr}
\toprule
Method & $B_{\mathrm{benign}}$ & $B_{\mathrm{zir}}$ & $B_{\mathrm{skew}}$
       & $B_{\mathrm{cv}}$ & $B$ & RBV & WMAPE \\
\midrule
Croston        & $ 0.098$ & $ 0.077$ & $ 0.086$ & $ 0.100$ & $ 0.087$ & 0.020 & 0.739 \\
SBA            & $ 0.043$ & $ 0.023$ & $ 0.031$ & $ 0.045$ & $ 0.033$ & 0.019 & 0.725 \\
TSB            & $ 0.095$ & $ 0.062$ & $ 0.076$ & $ 0.061$ & $ 0.075$ & 0.028 & 0.734 \\
TSB-HB         & $ 0.057$ & $ 0.035$ & $ 0.036$ & $ 0.059$ & $ 0.043$ & 0.024 & 0.739 \\
\midrule
Oracle mean    & $ 0.068$ & $ 0.026$ & $ 0.035$ & $ 0.023$ & $ 0.040$ & 0.036 & 0.738 \\
Oracle median  & $ 0.001$ & $-0.456$ & $-0.390$ & $-0.941$ & $-0.353$ & 0.695 & 0.668 \\
\bottomrule
\end{tabular}
\end{table}

\subsection{Zero-Shot Foundation-Model Validation}
\label{app:chronos}

Building on the same protocol, we test whether the characterized
anatomy survives the removal of gradient training entirely. We evaluate
a frozen Chronos-Bolt checkpoint~\cite{ansari2024chronos} in zero-shot
mode---no parameter update on M5---under the identical evaluation
convention (regime map, sum-ratio bias, unweighted RBV). The released
mean head degenerates to the $0.5$-quantile in our build, so the mean
estimand is recovered by trapezoid quadrature over 50 quantile
heads, $\hat\mu = \int_0^1 \hat{Q}(u)\,du$, $u = 0.01, 0.03, \dots, 0.99$; the median
estimand is the $q_{0.5}$ head. The two estimands therefore share one
checkpoint and one forward pass, isolating the estimand from capacity
and training effects.

\begin{table}[t]
\centering
\caption{Chronos-Bolt zero-shot on the M5 single-origin protocol
(App.~\ref{app:m5_exp_setup}), horizons 1--7 pooled ($n{=}187{,}222$ samples).
Mean vs.\ median estimands decoded from the \emph{same} frozen forward
pass. Bias $=\sum(\hat y-y)/\sum|y|$ per cell; RBV and WMAPE as in
Tab.~\ref{tab:regime_perf}.}
\label{app:tab:chronos_m5}
\small
\begin{tabular}{lrrrrrrr}
\toprule
estimand & $B_{\rm benign}$ & $B_{\rm zir}$ & $B_{\rm skew}$ & $B_{\rm cv}$
& $B_{\rm global}$ & RBV & WMAPE \\
\midrule
mean (quadrature) & $0.114$ & $-0.119$ & $-0.060$ & $-0.328$ & $0.008$
& $\mathbf{0.381}$ & $0.703$ \\
median ($q_{0.5}$) & $0.070$ & $-0.493$ & $-0.368$ & $-0.853$ & $-0.194$
& $0.789$ & $\mathbf{0.672}$ \\
\bottomrule
\end{tabular}
\end{table}

\begin{figure*}[t]
\centering
\includegraphics[width=\linewidth]{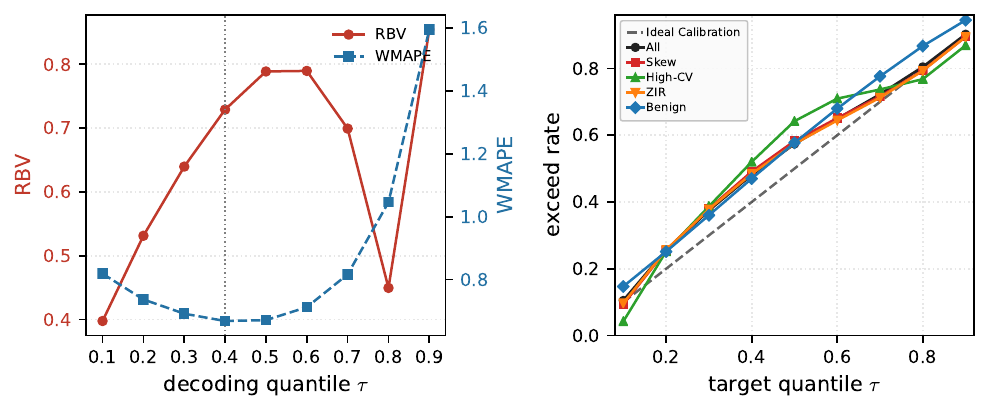}
\caption{\textbf{Chronos-Bolt zero-shot evaluation on M5.} \emph{Left:} RBV
and aggregate WMAPE as the decoding quantile $\tau$ is swept on a frozen
checkpoint; the dotted vertical line marks $\tau^{\ast}{=}0.4$, the quantile
minimizing aggregate WMAPE. RBV($\tau$) is U-shaped and attains its minimum
at the quadrature-mean estimand (Tab.~\ref{app:tab:chronos_m5}).
\emph{Right:} per-regime exceed rate,
$\Pr[\hat{y}>y]+0.5\Pr[\hat{y}=y]$, of zero-shot quantile decoding against
the diagonal: quantile heads over-cover throughout the mid range (exceed
$0.574$ at nominal $0.5$), the high-CV segment shows the largest
over-coverage, and calibration recovers at high $\tau$. Unlike
Fig.~\ref{fig:quantile_calibration}, the low-$\tau$ end shows no
zero-inflation floor---the continuous quantile heads emit no atom at zero,
so the ZIR curve tracks the diagonal while the floor in Fig.~\ref{fig:fig_m5_quantile_calibration} is a
construct of discrete zero-mass predictors. Negative predictions ($65\%$ at $q_{0.1}$) mechanically depress the
exceed rate at low $\tau$ (App.~\ref{app:chronos}).}
\label{app:fig:chronos_tau}
\end{figure*}

Four observations follow. (i)~\emph{Estimand, not training, sets the
bias structure:} the median under-predicts globally ($-0.19$) with
monotone worsening across the benign$\to$zir$\to$CV hierarchy (down to
$-0.85$), whereas the quadrature mean is globally unbiased ($+0.01$) yet
conceals a $+0.11/-0.33$ benign/CV split---unbiasedness by cancellation,
the signature of Fig.~\ref{fig:paradigm_evolution}. (ii)~\emph{The
accuracy--miscalibration trade-off persists:} the median attains the
better WMAPE but twice the regime-bias dispersion (RBV $0.79$ vs.\
$0.38$), mirroring the quantile-family trade-off of
Tab.~\ref{tab:regime_perf}. (iii)~\emph{The $\tau$-scan anatomy
reproduces:} sweeping $\tau \in \{0.1,\dots,0.9\}$ with each $q_\tau$ as a
point forecast yields a U-shaped RBV($\tau$) with WMAPE-optimal
$\tau^{\ast}{=}0.4<0.5$ (Fig.~\ref{app:fig:chronos_tau}), and the
quadrature mean sits at the RBV minimum; the median head is miscalibrated
in coverage (exceed rate $0.574$ at nominal $0.5$), mirroring
Fig.~\ref{fig:quantile_calibration}. (iv)~\emph{The quantile heads are unconstrained and emit inadmissible
negative demand:} $65.2\%$ of $q_{0.1}$ and $5.0\%$ of $q_{0.5}$
(sku,~horizon) predictions are negative, whereas the quadrature mean is
effectively non-negative ($0.008\%$). This mechanically depresses the
exceed rate at low $\tau$---a negative forecast can never exceed zero
demand---so the apparent low-$\tau$ calibration of the quantile curves
reflects this artifact rather than quantile accuracy, and part of the
median's under-prediction in pathological cells stems from inadmissible
values rather than from the estimand alone. The quadrature mean, which is
non-negative by construction, isolates the estimand effect cleanly.





\end{document}